\documentclass[english]{article}
\usepackage{nameref,zref-xr}
\usepackage{array}
\usepackage{float}
\usepackage{multirow}
\usepackage{amsmath}
\usepackage{amsthm}
\usepackage{amssymb}
\usepackage{graphicx}
\usepackage{xargs}[2008/03/08]
\PassOptionsToPackage{normalem}{ulem}
\usepackage{ulem}
\usepackage{wrapfig}
\usepackage{nameref,zref-xr}
\usepackage{hyperref}       
\usepackage{url}            
\usepackage{stmaryrd}
\usepackage{booktabs}       
\usepackage{amsfonts}       
\usepackage{nicefrac}       
\usepackage{microtype}      
\usepackage[dvipsnames]{xcolor}
\usepackage{algorithmic}
\usepackage[normalem]{ulem}
\usepackage{cleveref}

\crefformat{footnote}{#2\footnotemark[#1]#3}

\usepackage{mathtools}
\usepackage{babel}
\providecommand{\lemmaname}{Lemma}
\providecommand{\theoremname}{Theorem}

 \normalsize
\makeatletter

\providecommand{\tabularnewline}{\\}
\floatstyle{ruled}
\newfloat{algorithm}{tbp}{loa}
\providecommand{\algorithmname}{Algorithm}
\floatname{algorithm}{\protect\algorithmname}

\theoremstyle{plain}
\newtheorem{thm}{\protect\theoremname}
\theoremstyle{plain}
\newtheorem{lem}[thm]{\protect\lemmaname}

\@ifundefined{showcaptionsetup}{}{%
 \PassOptionsToPackage{caption=false}{subfig}}
\usepackage{subfig}
\makeatother

\usepackage{babel}
\providecommand{\lemmaname}{Lemma}
\providecommand{\theoremname}{Theorem}

 \normalsize

    \usepackage[final, nonatbib]{neurips_2022}

\author{
    Hoang Phan$^{1}$ \hspace{0.09cm}
    Ngoc N. Tran$^{1}$ \hspace{0.09cm}
    Trung Le$^{2}$ \hspace{0.09cm} 
    Toan Tran$^{1}$ \hspace{0.09cm}
     Nhat Ho$^{3}$ \hspace{0.09cm}
    Dinh Phung$^{1,2}$ \vspace{3mm}\\
    $^{1}$ VinAI Research, Vietnam \\
    $^{2}$ Monash University, Australia \\
    $^{3}$ University of Texas, Austin
}

\usepackage{amsmath,amsfonts,bm}

\def\eqref#1{equation~\ref{#1}}

\def\1{\bm{1}}

\DeclareMathAlphabet{\mathsfit}{\encodingdefault}{\sfdefault}{m}{sl}
\SetMathAlphabet{\mathsfit}{bold}{\encodingdefault}{\sfdefault}{bx}{n}

\DeclareMathOperator*{\argmin}{arg\,min}

\global\long\def\goto{\rightarrow}%

\global\long\def\and{\cap}%

\global\long\def\simplex{\Delta}%

\global\long\def\ess{\mathbb{E}}%

\newcommandx\ESS[2][usedefault, addprefix=\global, 1=]{\underset{#2}{\ess}\left[#1\right]}%

\global\long\def\trace{\mathrm{tr}}%

\global\long\def\argmin#1{\underset{_{#1}}{\text{argmin}\ } }%

\global\long\def\norm{}%

\global\long\def\norm#1{\left\Vert #1\right\Vert }%

\begin{document}
\title{Stochastic Multiple Target Sampling Gradient Descent}

\maketitle

\begin{abstract}
Sampling from an unnormalized target distribution is an essential problem
with many applications in probabilistic inference. Stein Variational
Gradient Descent (SVGD) has been shown to be a powerful method that iteratively updates a set of particles to approximate the distribution of interest.
Furthermore, when analysing its asymptotic properties, SVGD reduces exactly to a single-objective optimization problem and can be viewed as a probabilistic version of this single-objective optimization
problem. A natural question then arises: ``\emph{Can we derive a probabilistic version of the multi-objective optimization?}''. To answer
this question, we propose \emph{Stochastic Multiple Target Sampling Gradient Descent}
(MT-SGD), enabling us to sample from multiple unnormalized target
distributions. Specifically, our MT-SGD conducts a flow of intermediate
distributions gradually orienting to multiple target distributions,
which allows the sampled particles to move to the joint high-likelihood
region of the target distributions. Interestingly, the asymptotic
analysis shows that our approach reduces exactly to the multiple-gradient
descent algorithm for multi-objective optimization, as expected. Finally,
we conduct comprehensive experiments to demonstrate the merit of our
approach to multi-task learning.
\end{abstract}

\section{Introduction}

Sampling from an unnormalized target distribution that we know the
density function up to a scaling factor is a pivotal problem with
many applications in probabilistic inference \cite{bishop,murphy,MAL-001}.
For this purpose, Markov chain Monte Carlo (MCMC) has been widely
used to draw approximate posterior samples, but unfortunately, is often time-consuming and has difficulty accessing the convergence \cite{liu2016stein}.
Targeting an efficient acceleration of MCMC, some stochastic variational particle-based
approaches have been proposed, notably Stochastic Langevin Gradient
Descent \cite{welling2011bayesian} and Stein Variational Gradient
Descent (SVGD) \cite{liu2016stein}. Outstanding among them is SVGD,
with a solid theoretical guarantee of the convergence of the set of
particles to the target distribution by maintaining a flow of distributions. More specifically,
SVGD starts from an arbitrary and easy-to-sample initial distribution
and learns the subsequent distribution in the flow by push-forwarding the
current one using a function $T\left(x\right)=x+\epsilon\phi\left(x\right)$,
where $x\in\mathbb{R}^{d}$, $\epsilon>0$ is the learning rate, and
$\phi\in\mathcal{H}_{k}^{d}$ with $\mathcal{H}_{k}$ to be the Reproducing
Kernel Hilbert Space corresponding to a kernel $k$. It is
well-known that for the case of using Gaussian RBF kernel, by letting the kernel width approach $+\infty$,
the update formula of SVGD  at each step  asymptotically reduces to the typical gradient
descent (GD) \cite{liu2016stein}, showing the connection between a probabilistic
framework like SVGD and a single-objective optimization algorithm. In other
words, SVGD can be viewed as a probabilistic version of the GD for
single-objective optimization.

On the other side, multi-objective optimization (MOO) \cite{desideri2012multiple} aims to
optimize a set of objective functions and manifests itself in many real-world applications
problems, such as 
in multi-task learning (MTL) \cite{mahapatra2020multi, sener2018multi},
natural language processing \cite{anderson2021modest}, and reinforcement learning \cite{ghosh2013towards, pirotta2016inverse, parisi2014policy}. Leveraging the above insights, it is natural to ask: ``\emph{Can we derive a probabilistic version of multi-objective optimization?}''. By answering this question, we enable the application of the Bayesian inference framework to the tasks inherently fulfilled by the MOO framework.  

\textbf{Contribution.} In this paper, we provide an affirmative answer to that question. In particular, we go beyond the SVGD to propose
\emph{Stochastic Multiple Target Sampling Gradient Descent} (MT-SGD), enabling us to sample
from multiple target distributions.  By considering the push-forward map $T\left(x\right)=x+\epsilon\phi\left(x\right)$
with $\phi\in\mathcal{H}_{k}^{d}$, we can find a closed-form for the
optimal push-forward map $T^{*}$ pushing the current distribution
on the flow simultaneously closer to all target distributions. Similar to SVGD, in the case of using Gaussian RBF kernel, when the kernel width approaches $+\infty$, MT-SGD reduces
exactly to the  multiple-gradient descent algorithm (MGDA) \cite{desideri2012multiple} for multi-objective optimization
(MOO). Our MT-SGD, therefore, can be considered as a probabilistic version of
the GD multi-objective optimization \cite{desideri2012multiple} as
expected.

Additionally, in practice, we consider a flow of {empirical} distributions, in which, each distribution is presented as a set of particles. Our observations indicate that MT-SGD globally drives
the particles to close to all target distributions, leading them
to diversify on the \emph{joint high-likelihood region} for all distributions.
It is worth noting that, different from other multi-particle approaches \cite{lin2019pareto,liu2021profiling,mahapatra2020multi}
leading the particles to diversify on a Pareto front, our MT-SGD orients
the particle to diversify on the so-called \emph{Pareto common} (i.e., the joint high-likelihood
region for all distributions) (cf. Section \ref{subsec:Com-to-MOO-SVGD}
for more discussions). We argue and empirically demonstrate that this characteristic is essential for the Bayesian setting, whose main goal is to estimate the \textit{ensemble accuracy} and the 
\textit{uncertainty calibration} of a model. In summary, we make the following contributions in this work:

\begin{itemize}
    \item Propose a principled framework that incorporates the power of Stein Variational Gradient Descent into multi-objective optimization. Concretely, our method is motivated by the theoretical analysis of SVGD, and we further derive the formulation that extends the original work and allows to sample from multiple unnormalize distributions.
    \item Demonstrate our algorithm is readily applicable in the context of multi-task learning. The benefits of MT-SGD are twofold: i) the trained network is optimal, which could not be improved in any task without diminishing another, and ii) there is no need for predefined preference vectors as in previous works \cite{lin2019pareto, mahapatra2020multi}, MT-SGD implicitly learns diverse models universally optimizing for all tasks.
    \item  Conduct comprehensive experiments to verify the behaviors of MT-SGD and demonstrate the superiority of MT-SGD to the baselines in a Bayesian setting, with higher ensemble performances and significantly lower calibration errors.
\end{itemize} 

\textbf{Related works.} The work of \cite{desideri2012multiple} proposed a multi-gradient
descent algorithm for multi-objective optimization (MOO) which opens
the door for the applications of MOO in machine learning and deep
learning. Inspired by \cite{desideri2012multiple}, MOO has been applied
in multi-task learning (MTL) \cite{mahapatra2020multi, sener2018multi},
few-shot learning \cite{chen2021pareto, ye2021multi}, and knowledge distillation
\cite{chennupati2021adaptive, du2020agree}. Specifically, in an earlier attempt at solving MTL, \cite{sener2018multi}
viewed multi-task learning as a multi-objective optimization problem,
where a task network consists of a shared feature extractor and a
task-specific predictor. In another study, \cite{mahapatra2020multi} developed
a  gradient-based multi-objective MTL algorithm to find a set of solutions
that satisfies the user preferences. Also follows the idea of learning neural networks conditioned on pre-defined preference
vectors, \cite{lin2019pareto}
proposed Pareto MTL, aiming to find a set of well-distributed Pareto solutions, which can represent different trade-offs among different tasks. Recently,
the work of \cite{liu2021profiling} leveraged MOO with SVGD \cite{liu2016stein} and Langevin dynamics \cite{welling2011bayesian} to diversify the solutions of MOO. In another line of work, \cite{ye2021multi} proposed a bi-level MOO that can be applied to few-shot learning. Furthermore, a somewhat different result was proposed, \cite{du2020agree} applied MOO to enable the knowledge distillation from multiple teachers and find a better optimization direction in  training the student network.

\textbf{Outline.} The paper is organized as follows. In Section \ref{sec:method}, we first present our theoretical contribution by reviewing the formalism and providing the point of view adopted to generalize SVGD in the context of MOO. Then, Section \ref{sec:mtl} introduces an algorithm to showcase the application of our proposed method in the multi-task learning scenario.  We report the results of extensive experimental studies performed
on various datasets that demonstrate the behaviors and efficiency of MT-SGD in Section \ref{sec:experiment}. Finally, we conclude the paper in Section \ref{sec:conclusion}. The complete proofs and experiment setups are deferred to the supplementary material.



\section{Multi-Target Sampling Gradient Descent \label{sec:method}}

We first briefly introduce the formulation of the multi-target sampling in Section \ref{subsec:problem_setting}. Second, Section \ref{subsec:theory} presents our theoretical development and shows how our proposed method is applicable to this problem. Finally, we detail how to train the proposed method in Section \ref{subsec:mt-sgd_algo} and highlight key differences between our method and related work in Section \ref{subsec:Com-to-MOO-SVGD}.

\subsection{Problem Setting}
\label{subsec:problem_setting}
Given a set of target distributions $p_{1:K}\left(\theta\right):=\{p_{1}(\theta),\ldots,p_{K}(\theta)\}$
with parameter $\theta\in\mathbb{R}^{d}$, we aim to find the optimal
distribution $q^{*}\in\mathcal{Q}$ that minimizes a vector-valued objective function whose $k$-th component is $D_{KL}(q\Vert p_{k})$:
\begin{equation}
\min_{q\in\mathcal{Q}}\left[D_{KL}\left(q\Vert p_{1}\right),...,D_{KL}\left(q\Vert p_{K}\right)\right],\label{eq:mul_dis_op}
\end{equation}
where $D_{KL}$ represents Kullback-Leibler divergence and $\mathcal{Q}$
is a family of distributions. 

When there exists an objective function that conflicts with each other, there will be a trade-off between these two objectives. Therefore, no “optimal solution” exists in such cases. Alternatively, we are often interested in seeking a set of solutions such that each does not have any better solution (i.e. achieves lower loss values in all objectives) \cite{lin2019pareto, liu2021profiling, sener2018multi}.  The optimization problem (OP) in (\ref{eq:mul_dis_op}) thus can be viewed as a multi-objective OP \cite{desideri2012multiple} on the probability distribution space. Let us denote $\mathcal{H}_{k}$ by the Reproducing Kernel Hilbert Space (RKHS) associated with a positive semi-definite (p.s.d.) kernel $k$, and $\mathcal{H}_{k}^{d}$ by the $d$-dimensional
vector function:
\begin{equation*}
f=[f_{1},\ldots,f_{d}], (f_{i}\in\mathcal{H}_{k}).
\end{equation*}
 Inspired by \cite{liu2016stein}, we construct a flow of distributions
$q_{0},q_{1},...,q_{L}$ departed from a simple distribution $q_{0}$,
that gradually move closer to all the target distributions. In
particular, at each step, assume that $q$ is the current obtained
distribution, and the goal is to learn a transformation $T=id+\epsilon\phi$
so that the \emph{feed-forward distribution} $q^{[T]}=T\#q$ moves
closer to $p_{1:K}$ simultaneously. Here we use $id$ to denote the
identity operator, $\epsilon>0$ is a step size, and $\phi\in\mathcal{H}_{k}^{d}$
is a velocity field. Particularly, the problem of finding the optimal
transformation $T$ {for the current step} is formulated as:
\begin{equation}
\min_{\phi}\left[D_{KL}\left(q^{[T]}\Vert p_{1}\right),...,D_{KL}\left(q^{[T]}\Vert p_{K}\right)\right].\label{eq:T_op}
\end{equation}
\subsection{Our Theoretical Development \label{subsec:theory}}

It is worth noting that the transformation $T$ defined above is injective
when $\epsilon$ is sufficiently small \cite{liu2016stein}. We consider
each $D_{KL}\left(q^{[T]}\Vert p_{i}\right),i=1,...,K$ as a function
w.r.t. $\epsilon$, by applying the first-order Taylor expansion at $0$, we have:
\[
D_{KL}\left(q^{[T]}\Vert p_{i}\right)=D_{KL}\left(q\Vert p_{i}\right)+\nabla_{\epsilon}D_{KL}\left(q^{[T]}\Vert p_{i}\right)\Big|_{\epsilon=0}\epsilon+O\left(\epsilon^{2}\right),
\]

where {$\lim_{\epsilon\goto0}O\left(\epsilon^2\right)/\epsilon^{2}=const$. }

Given that the velocity field $\phi\in\mathcal{H}_{k}^{d}$, similar to \cite{liu2016stein}, the gradient $\nabla_{\epsilon}D_{KL}\left(q^{[T]}\Vert p_{i}\right)\mid_{\epsilon=0}$
can be calculated as provided in \cref{footnote:supp} 
\[
\nabla_{\epsilon}D_{KL}\left(q^{[T]}\Vert p_{i}\right)\Big|_{\epsilon=0}=-\left\langle \phi,\psi_{i}\right\rangle _{\mathcal{H}_{k}^{d}},
\]
where $\psi_{i}\left(\cdot\right)=\mathbb{E}_{\theta\sim q}\left[k\left(\theta,\cdot\right)\nabla_{\theta}\log p_{i}\left(\theta\right)+\nabla_{\theta}k\left(\theta.\cdot\right)\right]$
and $\left\langle \cdot,\cdot\right\rangle _{\mathcal{H}_{k}^{d}}$
is the dot product in the RKHS.




This means that, for each target distribution $p_{i}$, the steepest
descent direction is $\phi_{i}^{*}=\psi_{i}$, in which the KL divergence
of interest $D_{KL}\left(q^{[T]}\Vert p_{i}\right)$ gets decreased
roughly by $-\epsilon\norm{\phi_{i}^{*}}_{\mathcal{H}_{k}^{d}}^{2}$
toward the target distribution $p_{i}$. However, this only guarantees
a divergence reduction for a single target distribution $p_{i}$ itself.
Our next aim is hence to find a common direction $\phi^{*}$ to reduce
the KL divergences w.r.t. all target distributions, which is reflected
in the following lemma, showing us how to combine the individual steepest
descent direction $\phi_{i}^{*}=\psi_{i}$ to yield the optimal direction
$\phi^{*}$ as summarized in Figure \ref{fig:find_common_phi}. 
\begin{figure}[!ht]
\begin{centering}
  \includegraphics[width=0.55\textwidth,trim=.75cm .5cm 1.5cm .7cm,clip]{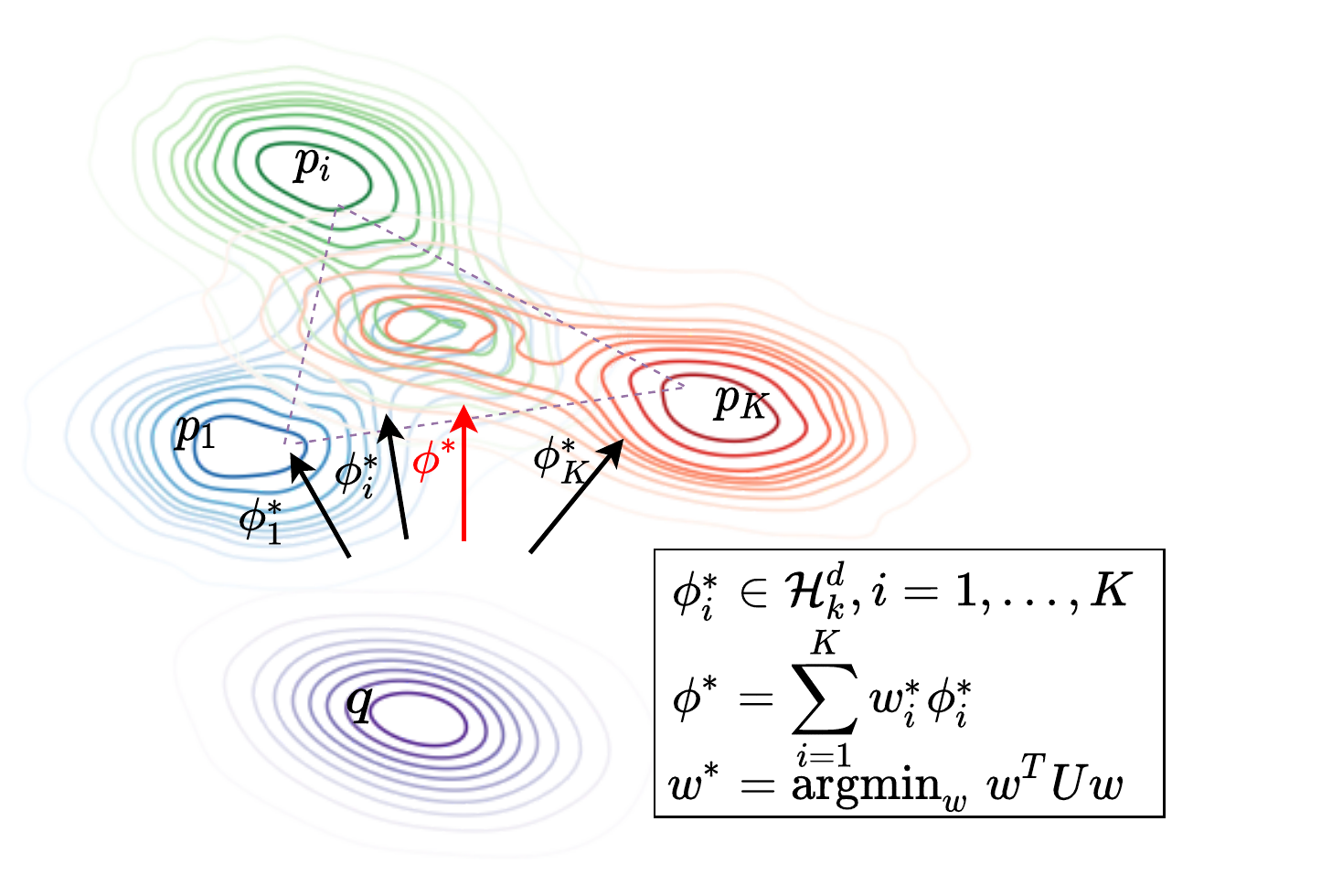}
\captionof{figure}{How to find the optimal descent direction $\phi^{*}$.\label{fig:find_common_phi}}  
\end{centering}
\end{figure}

\begin{lem}
\label{lem:common_direction}Let $w^{*}$ be the optimal solution of the optimization problem  $w^{*}=\argmin{w\in\simplex_{K}}w^{T}Uw$
and $\phi^{*}=\sum_{i=1}^{K}w_{i}^{*}\phi_{i}^{*}$, where $\simplex_{K}=\left\{ \pi\in\mathbb{R}_{+}^{K}:\norm{\pi}_{1}=1\right\} $
and $U\in\mathbb{R}^{K\times K}$ with $U_{ij}=\left\langle \phi_{i}^{*},\phi_{j}^{*}\right\rangle _{\mathcal{H}_{k}^{d}}$,
then we have
\[
\left\langle \phi^{*},\phi_{i}^{*}\right\rangle _{\mathcal{H}_{k}^{d}}\geq\norm{\phi^{*}}_{\mathcal{H}_{k}^{d}}^{2},i=1,...,K.
\]
\end{lem}
Lemma \ref{lem:common_direction} provides a common descent direction
$\phi^{*}$ so that all KL divergences w.r.t. the target distributions
are consistently reduced by $\epsilon\norm{\phi^{*}}_{\mathcal{H}_{k}^{d}}^{2}$
roughly and Theorem \ref{thm:KL_reduce} confirms this argument.
\begin{thm}
\label{thm:KL_reduce}If there does not exist $w\in\simplex_{K}$
such that $\sum_{i=1}^{K}w_{i}\phi_{i}^{*}=0$, given a sufficiently
small step size $\epsilon$, all KL divergences w.r.t. the target
distributions are strictly decreased by at least $A\norm{\phi^{*}}_{\mathcal{H}_{k}^{d}}^{2}>0$
where $A$ is a positive constant.
\end{thm}
The next arising question is how to evaluate the matrix $U\in\mathbb{R}^{K\times K}$
with $U_{ij}=\left\langle \phi_{i}^{*},\phi_{j}^{*}\right\rangle _{\mathcal{H}_{k}^{d}}$
for solving the quadratic problem: $\min_{w\in\simplex_{K}}w^{T}Uw$.
To this end, using some well-known equalities in the RKHS\footnote{\label{footnote:supp}All proofs and derivations can be found in the supplementary material.}, we arrive at the following formula: 
\begin{align}
U_{ij} & =\left\langle \phi_{i}^{*},\phi_{j}^{*}\right\rangle _{\mathcal{H}_{k}^{d}}=\mathbb{E}_{\theta,\theta'\sim q}\Biggl[k\left(\theta,\theta'\right)\left\langle \nabla\log p_{i}(\theta),\nabla\log p_{j}(\theta')\right\rangle \nonumber \\
 & +\left\langle \nabla\log p_{i}(\theta),\frac{\partial k(\theta,\theta')}{\partial\theta'}\right\rangle +\left\langle \nabla\log p_{j}(\theta'),\frac{\partial k(\theta,\theta')}{\partial\theta}\right\rangle +\trace\left(\frac{\partial^{2}k(\theta,\theta')}{\partial\theta\partial\theta'}\right)\Biggr],
 \label{eq:Uij}
\end{align}
where $\trace(\cdot)$ denotes the trace of a (square) matrix.

\subsection{Algorithm for MT-SGD \label{subsec:mt-sgd_algo}}

For the implementation of MT-SGD, we consider $q$ as a discrete distribution
over a set of $M$, ($M \in \mathbb N^{*}$) particles $\theta_1, \theta_2, \dots,\theta_M\sim q$. The formulation to evaluate
$U_{ij}$ in Equation. (\ref{eq:Uij}) becomes:
\begin{align}
U_{ij} & =\frac{1}{M^{2}}\sum_{a=1}^{M}\sum_{b=1}^{M}\Biggl[k\left(\theta_{a},\theta_{b}\right)\left\langle \nabla\log p_{i}(\theta_{a}),\nabla\log p_{j}(\theta_{b})\right\rangle +\left\langle \nabla\log p_{i}(\theta_{a}),\frac{\partial k(\theta_{a},\theta_{b})}{\partial\theta_{b}}\right\rangle \nonumber \\
 & +\left\langle \nabla\log p_{j}(\theta_{b}),\frac{\partial k(\theta_{a},\theta_{b})}{\partial\theta_{a}}\right\rangle +\trace\left(\frac{\partial^{2}k(\theta_{a},\theta_{b})}{\partial\theta_{a}\partial\theta_{b}}\right)\Biggr].\label{eq:Uij-1}
\end{align}

The optimal solution $\phi_{i}^{*}$ then can be computed as:
\begin{equation}
\phi_{i}^{*}\left(\cdot\right)=\frac{1}{M}\sum_{j=1}^{M}\left[k\left(\theta_{j},\cdot\right)\nabla_{\theta_{j}}\log p_{i}\left(\theta_{j}\right)+\nabla_{\theta_{j}}k\left(\theta_{j},\cdot\right)\right].\label{eq:opt_phi_dis}
\end{equation}
The key steps of our MT-SGD are summarized in Algorithm \ref{alg:alg},
where the set of particles $\theta_{1:M}$ is updated gradually to
approach the multiple distributions $p_{1:K}$. Furthermore, the update
formula consists of two terms: (i) the first term (i.e., relevant
to $k\left(\theta_{j},\cdot\right)\nabla_{\theta_{j}}\log p_{i}\left(\theta_{j}\right)$)
helps to push the particles to the \emph{joint high-likelihood region}
for all distributions and (ii) the second term (i.e., relevant to
$\nabla_{\theta_{j}}k\left(\theta_{j},\cdot\right)$) which is a \emph{repulsive
term} to push away the particles when they reach out each other. Finally,
we note that our proposed MT-SGD can be applied in the context where
we know the target distributions up to a scaling factor (e.g., in
the posterior inference).
\begin{algorithm}[H]
\begin{algorithmic}[1]
\REQUIRE Multiple unnormalized target densities $p_{1:K}$.
\ENSURE The optimal particles $\theta_1, \theta_2, \dots, \theta_M$.
\STATE  Initialize a set of particles $\theta_1, \theta_2, \dots, \theta_M\sim q_{0}$ .
\FOR {$t=1$ to $L$}
\STATE Form the matrix $U\in\mathbb{R}^{K\times K}$ with the element
$U_{ij}$ computed as in Equation. (\ref{eq:Uij-1}).
\STATE Solve the QP $\min_{w\in\simplex_{K}}w^{T}Uw$ to find the
optimal weights $w^{*}\in\simplex_{K}${\small{}.}{\small\par}
\STATE Compute the optimal direction $\phi^{*}\left(\cdot\right)=\sum_{i=1}^{K}w_{i}^{*}\phi_{i}^{*}\left(\cdot\right)$,
where $\phi_{i}^{*}$ is defined in Equation. (\ref{eq:opt_phi_dis}).
\STATE Update $\theta_{i}=\theta_{i}+\epsilon\phi^{*}\left(\theta_{i}\right),i=1,...,K$.
\ENDFOR
\STATE \textbf{return} $\theta_1, \theta_2, \dots, \theta_M$.
\end{algorithmic}
\caption{Pseudocode for MT-SGD.\label{alg:alg}}
\end{algorithm}

\textbf{Analysis for the case of RBF kernel. } We now consider a radial basis-function (RBF) kernel of bandwidth $\sigma$: $k\left(\theta,\theta'\right)=\exp\left\{ -\norm{\theta-\theta'}^{2}/\left(2\sigma^{2}\right)\right\} $ and examine some asymptotic behaviors. 

\textbf{$\blacktriangleright$~General case:}
The elements of the matrix $U$ become
\begin{align*}
U_{ij} & =\mathbb{E}_{\theta,\theta'\sim q}\Biggl[\exp\left\{ \frac{-\norm{\theta-\theta'}^{2}}{2\sigma^{2}}\right\} \Biggl[\left\langle \nabla\log p_{i}(\theta),\nabla\log p_{j}(\theta')\right\rangle \\
 & +\frac{1}{\sigma^{2}}\left\langle \nabla\log p_{i}(\theta)-\nabla\log p_{j}(\theta'),\theta-\theta'\right\rangle -\frac{d}{\sigma^{2}}-\frac{\norm{\theta-\theta'}^{2}}{\sigma^{4}}\Biggr]\Biggr].
\end{align*}

\textbf{$\blacktriangleright$\textmd{~}Single particle distribution $q=\delta_{\theta}$:} The elements of the matrix $U$ become
$$
U_{ij}=\left\langle \nabla\log p_{i}(\theta),\nabla\log p_{j}(\theta)\right\rangle ,
$$
and our formulation reduces exactly to MOO in \cite{desideri2012multiple}.

\textbf{$\blacktriangleright$~When $\sigma\protect\goto\infty$:}
The elements of the matrix $U$ become
\begin{align*}
U_{ij} & =\mathbb{E}_{\theta,\theta'\sim q}\Biggl[\left\langle \nabla\log p_{i}(\theta),\nabla\log p_{j}(\theta')\right\rangle \Biggr].
\end{align*}

\subsection{Comparison to MOO-SVGD and Other Works\label{subsec:Com-to-MOO-SVGD}}
The most closely related work to ours is MOO-SVGD \cite{liu2021profiling}. In
a nutshell, ours is principally different from that work and we 
show a fundamental difference between our MT-SGD and MOO-SVGD in Figure \ref{fig:Ours_Liu}.
Our MT-SGD navigates the particles from one distribution to another
distribution consecutively with a theoretical guarantee of globally
getting closely to multiple target distributions. By contrast, while MOO-SVGD
also uses the MOO \cite{desideri2012multiple} to update the particles
, their employed repulsive term encourages the particle diversity without any theoretical-guaranteed principle to control the repulsive term, hence it can force the particles to scatter on the multiple distributions. In fact, they aim to profile the whole Pareto front, which is preferred when users want to obtain a collection of diverse Pareto optimal solutions with different trade-offs among all tasks. 


Furthermore, it expects that our MT-SGD globally moves the set of
particles to the \emph{joint high-likelihood region} for all target
distributions. Therefore, we do not claim our MT-SGD as a method to
diversify the solution on a Pareto front for user preferences, as in
\cite{liu2021profiling, mahapatra2020multi}. Alternatively, our MT-SGD
can generate diverse particles on the so-called \emph{Pareto common}
(i.e., the joint high-likelihood region for all target distributions).
We argue and empirically demonstrate that by finding and diversifying
the particles on Pareto common for the multiple posterior inferences,
our MT-SGD can outperform the baselines on Bayesian-inference metrics
such as the ensemble accuracy and the calibration error.

\begin{figure}[!ht]
\begin{centering}
\includegraphics[width=1\textwidth]{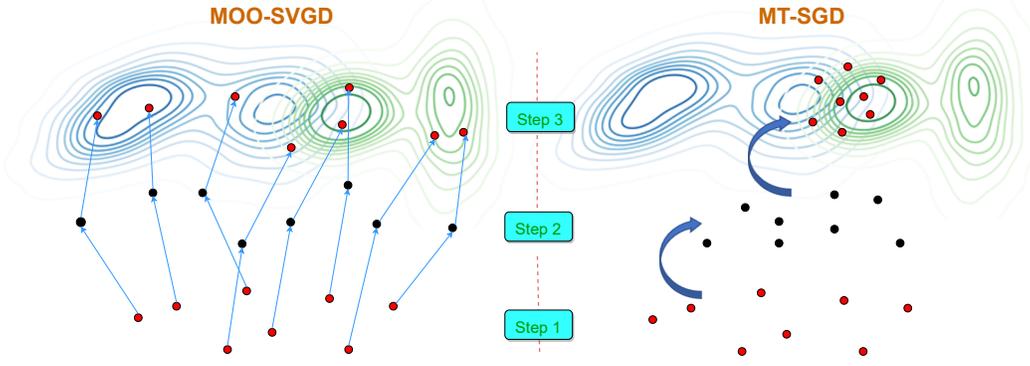}
\par\end{centering}
\caption{ Our MT-SGD moves the particles from one distribution to another distribution
to globally get closer to two target distributions (i.e., the blue
and green ones). Differently, MOO-SVGD uses MOO \cite{desideri2012multiple}
to move the particles individually and independently. The diversity
is enforced by the repulsive forces among particles. There is no principle
to control these repulsive forces, hence they can push the particles
scattering on two distributions. \label{fig:Ours_Liu}}
\end{figure}

Moreover, MOO-SVGD is \emph{not computationally efficient} when the number of particles is high because it requires solving an independent quadratic programming problem for each particle (cf. Section \ref{subsec:Sam_Mul_Syn}
and Figure \ref{fig:Sam_Mul_Syn} for the experiment on a synthetic dataset). Instead, our work solves a single  quadratic programming problem for all particles, then update them accordingly. We verify this computational improvement and present the training time of all baselines in the supplement material.

\section{Application to Multi-Task Learning \label{sec:mtl}}

For multi-task learning, we assume to have $K$ tasks $\left\{ \mathcal{T}_{i}\right\} _{i=1}^{K}$
and a training set $\mathbb{D}=\left\{ \left(x_{i},y_{i1},...,y_{iK}\right)\right\} _{i=1}^{N}$,
where $x_{i}$ is a data example and $y_{i1},...,y_{iK}$ are the
labels for the tasks. The model for each task $\theta^{j}=\left[\alpha,\beta^{j}\right],j=1,...,K$
consists of the \emph{shared part} $\alpha$ and \emph{non-shared}
part $\beta^{j}$ targeting the task $j$. The posterior $p\left(\theta^{j}\mid\mathbb{D}\right)$
for each task reads
{\begin{align*}
p\left(\theta^{j}\mid\mathbb{D}\right)&\propto p\left(\mathbb{D}\mid\theta^{j}\right)p\left(\theta^{j}\right)\propto \prod_{i=1}^{N}p\left(y_{ij}\mid x_{i},\theta^{j}\right)\\
&\propto\prod_{i=1}^{N}\exp\left\{ -\ell\left(y_{ij},x_{i};\theta^{j}\right)\right\} =\exp\left\{ -\sum_{i=1}^{N}\ell\left(y_{ij},x_{i};\theta^{j}\right)\right\},
\end{align*}
where $\ell$ is a loss function and the predictive likelihood $p(y_{ij}\mid x_i, \theta^j) \propto \exp\left\{ -\ell\left(y_{ij},x_{i};\theta^{j}\right)\right\}$ is examined. Note that the prior $p\left(\theta^{j}\right)$ here is retained from previous studies \cite{lin2019pareto,liu2021profiling}, which is a uniform and non-informative prior and can be treated as a constant term in our formulation.}

For our approach, we maintain a set of models $\theta_{m}=\left[\theta_{m}^{j}\right]_{j=1}^{K}$
with $m=1,...,M$, where $\theta_{m}^{j}=\left[\alpha_{m},\beta_{m}^{j}\right]$.
At each iteration, given the non-shared parts $\left[\beta^{j}\right]_{j=1}^{K}$
with $\beta^{j}=\left[\beta_{m}^{j}\right]_{m=1}^{M}$, we sample
the shared parts from the multiple distributions $p\left(\alpha\mid\beta^{j},\mathbb{D}\right),j=1,...,K$ as
\begin{equation}
\alpha_{m}\sim p\left(\alpha\mid\beta^{j},\mathbb{D}\right)\propto p\left(\alpha,\beta^{j}\mid\mathbb{D}\right)\propto p\left(\theta^{j}\mid\mathbb{D}\right)\propto\exp\left\{ -\sum_{i=1}^{N}\ell\left(y_{ij},x_{i};\theta^{j}\right)\right\} .\label{eq:shared_pos-1}
\end{equation}

We now apply our proposed MT-SGD to sample the shared parts $\left[\alpha_{m}\right]_{m=1}^{M}$
from the multiple distributions defined in (\ref{eq:shared_pos-1})
as
\begin{equation}
\alpha_{m}=\alpha_{m}+\epsilon\sum_{j=1}^{K}w_{j}^{*}\phi_{j}^{*}\left(\alpha_{m}\right),\label{eq:update_share}
\end{equation}
where $\phi_{j}^{*}\left(\alpha_{m}\right)=\frac{1}{M}\sum_{t=1}^{M}\left[k_{1}\left(\alpha_{t},\alpha_{m}\right)\nabla_{\alpha_{t}}\log p\left(\alpha_{t}\mid\beta^{j}_t,\mathbb{D}\right)+\nabla_{\alpha_{t}}k_{1}\left(\alpha_{t},\alpha_{m}\right)\right]$
and {$w^{*}=\left[w_{k}^{*}\right]_{k=1}^{K}$} are the weights received
from solving the quadratic programming problem. Here we note that
$\nabla_{\alpha_{t}}\log p\left(\alpha_{t}\mid\beta^{j}_t,\mathbb{D}\right)$
can be estimated via the batch gradient of the loss using Equation (\ref{eq:shared_pos-1}).

Given the updated shared parts $\left[\alpha_{m}\right]_{m=1}^{M}$,
for each task $j$, we update the corresponding non-shared parts $\left[\beta_{m}^{j}\right]_{m=1}^{M}$
by sampling
\begin{equation}
\beta_{m}^{j}\sim p\left(\beta^{j}\mid\alpha,\mathbb{D}\right)\propto p\left(\beta^{j},\alpha\mid\mathbb{D}\right)\propto p\left(\theta^{j}\mid\mathbb{D}\right)\propto\exp\left\{ -\sum_{i=1}^{N}\ell\left(y_{ij},x_{i};\theta^{j}\right)\right\} .\label{eq:non_shared}
\end{equation}

We now apply SVGD \cite{liu2016stein} to sample the non-shared parts
$\left[\beta_{m}^{j}\right]_{m=1}^{M}$ for each task $j$ from the
distribution defined in (\ref{eq:non_shared}) as
\begin{equation}
\beta_{m}^{j}=\beta_{m}^{j}+\epsilon\psi_{j}^{*}\left(\beta_{m}^{j}\right),\label{eq:update_non_share}
\end{equation}
where $\psi_{j}^{*}\left(\beta_{m}^{j}\right)=\frac{1}{M}\sum_{a=1}^{M}\left[k_{2}\left(\beta_{a}^{j},\beta_{m}^{j}\right)\nabla_{\beta_{a}^{j}}\log p\left(\beta_{a}^{j}\mid\alpha_a,\mathbb{D}\right)+\nabla_{\beta_{a}^{j}}k_{2}\left(\beta_{a}^{j},\beta_{m}^{j}\right)\right]$ with which the term $\nabla_{\beta_{a}^{j}}\log p\left(\beta_{a}^{j}\mid\alpha_a,\mathbb{D}\right)$
can be estimated via the batch loss gradient using Equation (\ref{eq:non_shared}). 

\begin{algorithm}[!ht]
\begin{algorithmic}[1]

\REQUIRE A training set $\mathbb{D}=\left\{ \left(x_{i},y_{i1},...,y_{iK}\right)\right\} _{i=1}^{N}$.

\ENSURE The models $\theta_{m}=\left[\theta_{m}^{j}\right]_{j=1}^{K}$
with $m=1,...,M$, where $\theta_{m}^{j}=\left[\alpha_{m},\beta_{m}^{j}\right]$.

\STATE  Initialize a set of particles $\theta_{1:M}\sim q_{0}$ .

\FOR {$epoch=1$ to $\#epoch$}

\FOR {$iter=1$ to $\#iter$}

\STATE Update the shared parts $\left[\alpha_{m}\right]_{m=1}^{M}$
using Equation. (\ref{eq:update_share}).

\FOR {$j=1$ to $K$}

\STATE Update the non-shared part $\left[\beta_{m}^{j}\right]_{m=1}^{M}$
using Equation. (\ref{eq:update_non_share}).

\ENDFOR

\ENDFOR

\ENDFOR

\STATE \textbf{return} $\theta_{1:M}$.

\end{algorithmic}

\caption{Pseudocode for multi-task learning MT-SGD.\label{alg:alg-multi-task}}
\end{algorithm}

Algorithm \ref{alg:alg-multi-task} summarizes the key steps of our
multi-task MT-SGD. Basically, we alternatively update the shared parts
given the non-shared ones and vice versa.

\section{Experiments  \label{sec:experiment}}

In this section, we verify our MT-SGD by evaluating its performance on both synthetic and real-world datasets. For our experiments, we use the RBF kernel $k\left(\theta,\theta'\right)=\exp\left\{ -\norm{\theta-\theta'}_{2}^{2}/\left(2\sigma^{2}\right)\right\} $.
The detailed training and configuration are given in the supplementary
material. Our codes are available
at \url{https://github.com/VietHoang1512/MT-SGD}.

\subsection{Experiments on Toy Datasets}

\subsubsection{Sampling from Multiple Distributions \label{subsec:Sam_Mul_Syn}}

We first qualitatively analyze the behavior of the proposed method
on sampling from three target distributions. Each target distribution
is a mixture of two Gaussians as $p_{i}\left(\theta\right)=\pi_{i1}\mathcal{N}\left(\theta\mid\mu_{i1},\Sigma_{i1}\right)+\pi_{i2}\mathcal{N}\left(\theta\mid\mu_{i2},\Sigma_{i2}\right)$
($i=1,2,3)$ where the mixing proportions $\pi_{i1}=0.7,\forall i$,
$\pi_{i2}=0.3,\forall i$, the means $\mu_{11}=\left[4,-4\right]^{T},\mu_{12}=\left[0,0.5\right]^{T}$,
$\mu_{21}=\left[-4,4\right]^{T},\mu_{22}=\left[0.5,0\right]^{T}$,
and $\mu_{31}=\left[-3,-3\right]^{T},\mu_{32}=\left[0,0\right]^{T}$,
and the common covariance matrix $\Sigma_{ij}=\left[\begin{array}{cc}
0.5 & 0\\
0 & 0.5
\end{array}\right],i=1,2,3$ and $j=1,2$. It can be seen from Figure \ref{fig:Sam_Mul_Syn} that
there is a common high-density region spreading around the origin.
The fifty particles are drawn randomly in the space, and the initialization
is retained across experiments for a fair comparison. 

\begin{figure}[H]
\centering
\includegraphics[width=1\columnwidth]{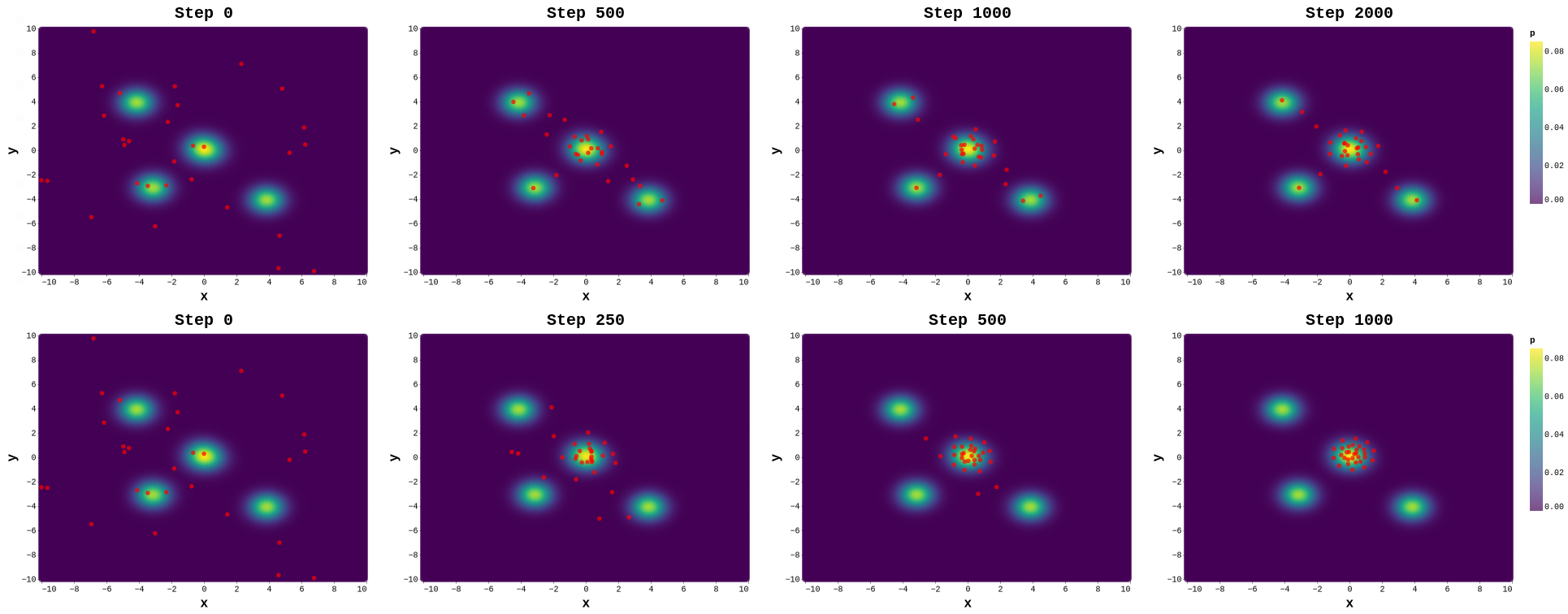}
\caption{Sampling from three mixtures of two Gaussian distributions with a joint
high-likelihood region. We run MOO-SVGD (top) and MT-SGD (bottom) to update the initialized
particles (left-most figures) until convergence using Adam optimizer
\cite{kingma2014adam}. While MOO-SVGD transports
 the initialized particles scattering on the distributions,  MT-SGD perfectly drives
them to diversify in the region of interest.\label{fig:Sam_Mul_Syn}}
\end{figure}
\vspace{-2mm}
Figure \ref{fig:Sam_Mul_Syn} shows the updated particles by MOO-SVGD
and MT-SGD at selected iterations, we observe that the particles from
MOO-SVGD spread out and tend to characterize all the modes, some of
them even scattered along trajectories due to the conflict in optimizing
multiple objectives. By contrast, our method is able to find and cover
the common high density region among target distributions with well-distributed
particles, which illustrates the basic principles of MT-SGD. Additionally,
at the $1,000$-th step, the training time for ours is 0.23 min, whereas
that for MOO-SVGD is 1.63 min. The reason is that MOO-SVGD requires solving an independent quadratic programming problem for each particle
at each step.

\subsubsection{Multi-objective Optimization}

The previous experiment illustrates that MT-SGD can be used to sample
from multiple target distributions, we next test our method on the
other low-dimensional multi-objectives OP from \cite{zitzler2000comparison}.
In particular, we use the two objectives ZDT3, whose Pareto front
consists of non-contiguous convex parts, to show our method simultaneously
minimizes both objective functions. Graphically, the simulation results
from Figure \ref{fig:zdt} show the difference in the convergence
behaviors between MOO-SVGD and MT-SGD: the solution set achieved
by MOO-SVGD covers the entire Pareto front, while ours distributes
and diversifies on the three middle curves (mostly concentrated in
the middle curve) which are the Pareto common having low values for
two objective functions in ZDT3.

\begin{figure}[!ht]
\vspace{-2mm}
\centering
\includegraphics[width=1\columnwidth]{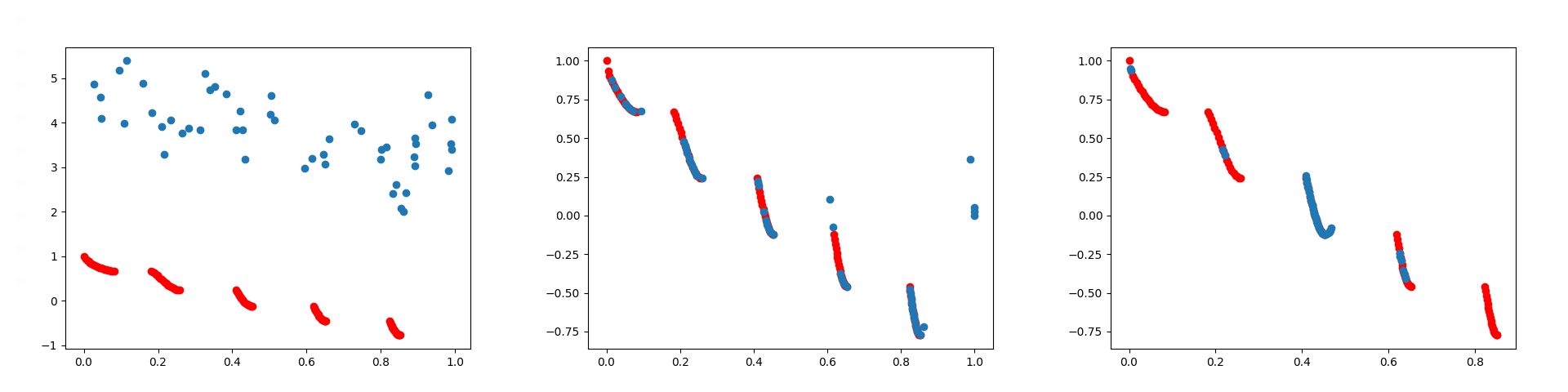}\caption{Solutions obtained by MOO-SVGD (mid) and MT-SGD (right) on ZDT3 problem
after 10,000 steps, with blue points representing particles and {red}
curves indicating the Pareto front. As expected, from initialized
particles (left), MOO-SVGD's solution set widely distributes on
the whole Pareto front while the one of MT-SGD concentrates
around middle curves (mostly the middle one). \label{fig:zdt} }
\end{figure}

\subsection{Experiments on Real Datasets}

\subsubsection{Experiments on Multi-Fashion+Multi-MNIST Datasets}

We apply the proposed MT-SGD method on multi-task learning, following Algorithm \ref{alg:alg-multi-task}. Our method is validated on
different benchmark datasets: (i) Multi-Fashion+MNIST \cite{NIPS2017_2cad8fa4},
(ii) Multi-MNIST, and (iii) Multi-Fashion. Each of them consists of
120,000 training and 20,000 testing images generated from MNIST
\cite{mnist} and FashionMNIST \cite{xiao2017fashion} by overlaying
an image on top of another: one in the top-left corner and one in
the bottom-right corner. Lenet \cite{mnist} (22,350 params) is employed
as the backbone architecture and trained for 100 epochs with SGD in this experimental setup. 

\begin{figure}[!ht]
\begin{centering}
\includegraphics[width=1\textwidth, trim=.0cm .0cm 0cm .0cm,clip]{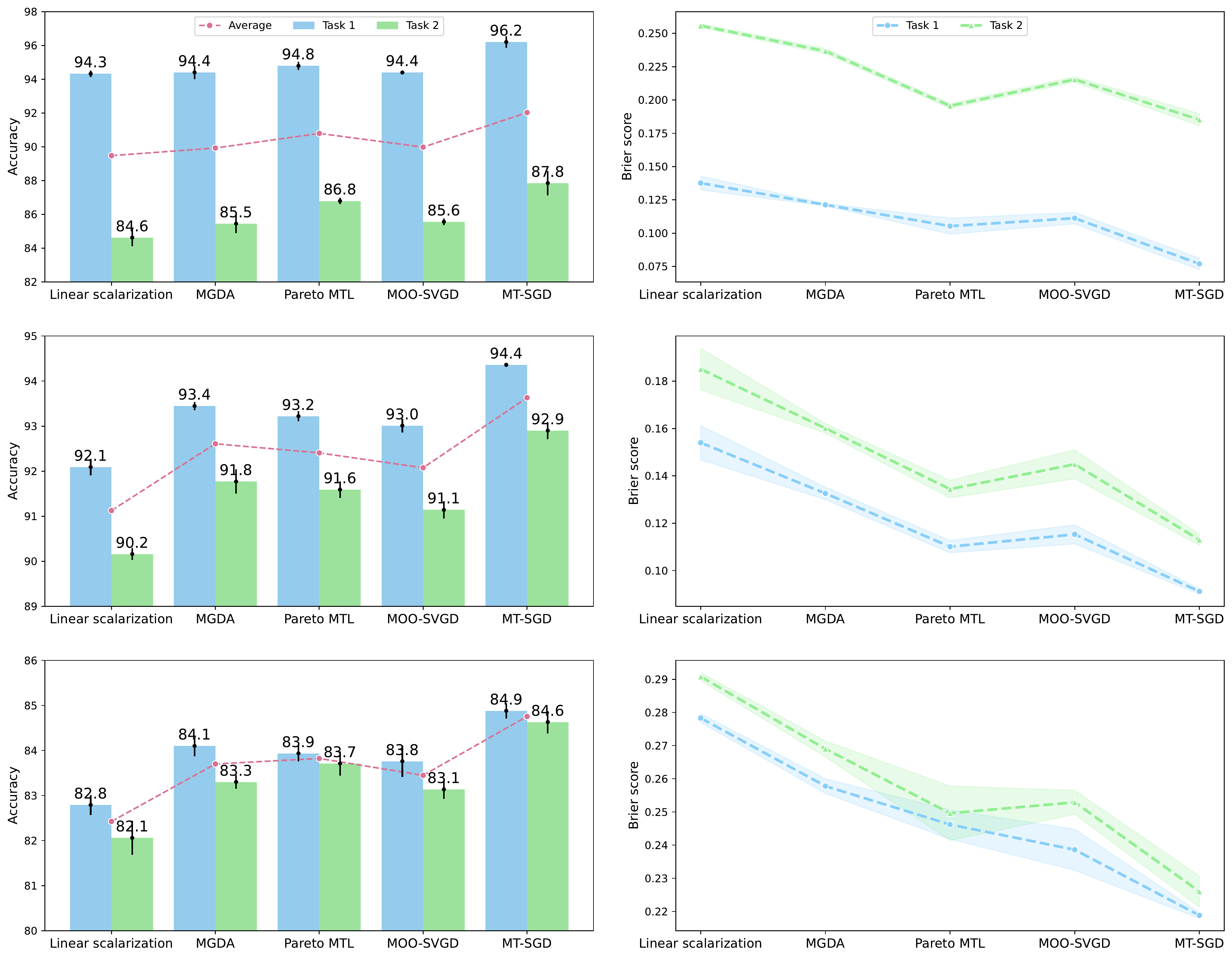}
\par\end{centering}
\caption{Results on Multi-Fashion+MNIST (left), Multi-MNIST (mid), and Multi-Fashion
(right). We report the ensemble accuracy (\emph{higher is better})
and the Brier score (\emph{lower is better}) {over 3 independent runs, as well as the standard deviation (the error bars and shaded areas in the figures)}.}
\label{fig:mnist-acc-brier}
\end{figure}

\textbf{Baselines: }In multi-task experiments, the introduced MT-SGD
is compared with state-of-the-art baselines including MGDA \cite{sener2018multi},
Pareto MTL \cite{lin2019pareto}, MOO-SVGD \cite{liu2021profiling}.
We note that to reproduce results for these baselines, we either use
the author's official implementation released on GitHub
or ask the authors for their codes. For MOO-SVGD and Pareto MTL, the
reported result is from the ensemble prediction of five particle models.
Additionally, for linear scalarization and MGDA, we train five particle
models independently with different initializations and then ensemble
these models. 

\textbf{Evaluation metrics: }We compare MT-SGD against baselines regarding both average accuracy and predictive uncertainty. Besides
the commonly used accuracy metric, we measure the quality and diversity
of the particle models by relying on two other popular Bayesian metrics: Brier
score \cite{brier1950verification, ovadia2019can} and expected
calibration error (ECE) \cite{dawid1982well, naeini2015obtaining}.

From Figure \ref{fig:mnist-acc-brier}, we observe that MT-SGD consistently
improves model performance across all tasks in both accuracy and Brier
score by large margins, compared to existing techniques in the literature.
The network trained using linear scalarization, as expected, produces
inferior ensemble results while utilizing MOO techniques helps yield
better performances. Overall, our proposed method surpasses the second-best
baseline by at least 1\% accuracy in any experiment. Furthermore,
Table \ref{tab:mnist-ece} provides a comparison between these methods
in terms of expected calibration error, in which MT-SGD also consistently
provides the lowest expected calibration error, illustrating our method's
ability to obtain well-calibrated models (the accuracy is closely
approximated by the produced confidence score). It is also worth noting
that while Pareto MTL has higher accuracy, MOO-SVGD produces slightly
better calibration estimation.
\begin{table}[th]

\caption{Expected calibration error (\%) (num\_bin $=10$) on Multi-MNIST, Multi-Fashion and Multi-Fashion+MNIST
datasets {over 3 runs}. We use the \textbf{bold} font to highlight the best results} \label{tab:mnist-ece}.
\vspace{1mm}
\centering{}\resizebox{1.0\textwidth}{!}{
\begin{tabular}{c|c|c|c|c|c|c}
\toprule
Dataset & Task & Linear scalarization & MGDA & Pareto MTL & MOO-SVGD & MT-SGD\tabularnewline
\midrule
\multirow{2}{*}{Multi-Fashion+MNIST} & Top left & 21.33 $\pm$ 0.83 & 19.91 $\pm$ 0.26 & 9.44 $\pm$ 0.65 & 9.47 $\pm$ 0.89 & \textbf{4.65} $\mathbf{\pm}$ \textbf{0.11} \tabularnewline
 & Bottom right & 17.76 $\pm$ 0.60 & 16.29 $\pm$ 1.35 & 4.73 $\pm$ 0.46 & 4.95 $\pm$ 0.49 & \textbf{3.17} $\mathbf{\pm}$ \textbf{0.20}\tabularnewline
\midrule
\multirow{2}{*}{Multi-MNIST} & Top left & 17.37 $\pm$ 0.62 & 15.29 $\pm$ 0.49 & 5.45 $\pm$ 0.85 & 5.37 $\pm$ 0.51 & \textbf{3.28} $\mathbf{\pm}$ \textbf{0.20}\tabularnewline
 & Bottom right & 18.09 $\pm$ 1.11 & 16.87 $\pm$ 0.67 & 7.34 $\pm$ 1.08 & 6.74 $\pm$ 0.50 & \textbf{4.00} $\mathbf{\pm}$ \textbf{0.19} \tabularnewline
\midrule
\multirow{2}{*}{Multi-Fashion} & Top left & 15.86 $\pm$ 1.20 & 14.48 $\pm$ 0.95 & 8.55 $\pm$ 0.69 & 5.48 $\pm$ 0.53 & \textbf{3.80} $\mathbf{\pm}$ \textbf{0.38}\tabularnewline
 & Bottom right & 15.98 $\pm$ 1.32 & 14.70 $\pm$ 1.63 & 9.01 $\pm$ 1.77 & 6.11 $\pm$ 0.54 & \textbf{4.47} $\mathbf{\pm}$ \textbf{0.21}\tabularnewline
\bottomrule
\end{tabular}}
\end{table}
\subsubsection{Experiment on CelebA Dataset}

In this experiment, we verify the significance of MT-SGD on a larger neural network:
Resnet18 \cite{he2016deep}, which consists of 11.4M parameters. We
take the first 10 binary classification tasks and randomly select a subset of 40k images
from the CelebA dataset \cite{liu2015deep}. Note that in this experiment, we consider
Single task, in which 10 models are trained separately and serves as a strong baseline.
\begin{table}[th]
\caption{Results on CelebA dataset, regarding accuracy
and expected calibration error. For the full names of the tasks, please
refer to our supplementary material. While MGDA trains a single model
only to adapt on all tasks, reported performance of MOO-SVGD and MT-SGD is the
ensemble results from five particle models.\label{tab:celeba-acc-ece}}
\vspace{1mm}
\centering{}\resizebox{1.0\textwidth}{!}{
\begin{tabular}{c|c|c|c|c|c|c|c|c|c|c|c|c}
\toprule
 & Method & 5S & AE & Att & BUE & Bald & Bangs & BL & BN & BlaH & BloH & Average\tabularnewline
\midrule 
\multirow{4}{*}{Acc (\%)} & Single task & 91.8 & 84.6 & \textbf{80.3} & 81.9 & 98.8 & 94.8 & 85.8 & 81.3 & 89.6 & 94.2 & 88.3\tabularnewline
 & MGDA & 91.8 & 84.0 & 79.0 & 81.3 & 98.6 & 94.6 & 83.6 & 81.6 & 89.8 & 93.8 & 87.8\tabularnewline
 & MOO-SVGD & 92.3 & 84.2 & 78.9 & 81.2 & 98.9 & 94.5 & \textbf{86.4} & 80.0 & 90.8 & 94.8 & 88.2\tabularnewline
 & MT-SGD & \textbf{92.6} & \textbf{84.8} & \textbf{80.3} & \textbf{82.9} & \textbf{99.1} & \textbf{95.2} & 86.3 & \textbf{82.6} & \textbf{91.1} & \textbf{95.0} & \textbf{89.0}\tabularnewline
\midrule 
\multirow{4}{*}{ECE (\%)} & Single task & 3.3 & 2.4 & 4.4 & 3.9 & 0.7 & 1.6 & 5.7 & 6.5 & 3.1 & 1.1 & 3.3\tabularnewline
 & MGDA & 1.4 & 1.1 & 3.5 & 7.3 & \textbf{0.3} & 1.8 & 6.9 & 5.4 & 2.1 & 1.2 & 3.1\tabularnewline
 & MOO-SVGD & 2.8 & 1.9 & 3.1 & 5.6 & \textbf{0.3} & \textbf{0.5} & \textbf{4.7} & 3.3 & \textbf{1.3} & 1.3 & 2.5\tabularnewline
 & MT-SGD & \textbf{1.2} & \textbf{1.4} & \textbf{1.7} & \textbf{2.3} & 0.6 & 1.7 & 6.8 & \textbf{1.2} & 2.1 & \textbf{0.9} & \textbf{2.0}\tabularnewline
\bottomrule

\end{tabular}}
\end{table}

The performance comparison of the mentioned models in CelebA experiment is shown in Table
\ref{tab:celeba-acc-ece}. As clearly seen from the upper part of
the table, MT-SGD performs best in all tasks, except in BL, where MOO-SVGD is slightly better (86.4\% vs 86.3\%). Moreover, our method matches or beats Single task - the second-best baseline in all tasks. Regarding the well-calibrated uncertainty estimates, ensemble learning methods exhibit better results. In particular, MT-SGD and MOO-SVGD provide the best calibration performances, which are $2\%$ and $2.5\%$, respectively, which emphasizes the importance of efficient ensemble
learning for enhanced calibration.



\section{Conclusion \label{sec:conclusion}}

In this paper, we propose Stochastic Multiple Target \textbf{S}ampling
Gradient Descent (MT-SGD), allowing us to sample the particles from
the joint high-likelihood of multiple target distributions. Our MT-SGD
is theoretically guaranteed to simultaneously reduce the divergences
to the target distributions. Interestingly, the asymptotic analysis
of our MT-SGD reduces exactly to the multi-objective optimization.
We conduct comprehensive experiments to demonstrate that by driving
the particles to the Pareto common (the joint high-likelihood of
multiple target distributions), our MT-SGD can outperform the baselines on the ensemble accuracy and the well-known Bayesian metrics such as the expected calibration error and the Brier score.

\clearpage{}

\bibliographystyle{plain}
\bibliography{ref}

\clearpage
\section*{Checklist}
\begin{enumerate}
\item For all authors... 
\begin{enumerate}
\item Do the main claims made in the abstract and introduction accurately
reflect the paper's contributions and scope? \answerYes{}
\item Did you describe the limitations of your work? \answerYes{In the supplementary material, we discuss the computational overhead of our proposed method, compared with related work.} 
\item Did you discuss any potential negative societal impacts of your work?
\answerNo{} 
\item Have you read the ethics review guidelines and ensured that your paper
conforms to them? \answerYes{}
\end{enumerate}
\item If you are including theoretical results... 
\begin{enumerate}
\item Did you state the full set of assumptions of all theoretical results?
\answerYes{We include assumptions in all theoretical statements.} 
\item Did you include complete proofs of all theoretical results? \answerYes{The main results and lemmas appear in the main paper and their complete proof is put in the supplementary material.} 
\end{enumerate}
\item If you ran experiments... 
\begin{enumerate}
\item Did you include the code, data, and instructions needed to reproduce
the main experimental results (either in the supplemental material
or as a URL)? \answerYes{They will be attached to the supplementary material.}
\item Did you specify all the training details (e.g., data splits, hyperparameters,
how they were chosen)? \answerYes{The training details are provided with the code, but the important details should be in the main paper.} 
\item Did you report error bars (e.g., with respect to the random seed after
running experiments multiple times)? \answerYes{We run each experiment over a number of random seeds and report the average result along with the standard deviation.} 
\item Did you include the total amount of compute and the type of resources
used (e.g., type of GPUs, internal cluster, or cloud provider)? \answerYes{The detailed training environment will be attached in the supplementary material.}
\end{enumerate}
\item If you are using existing assets (e.g., code, data, models) or curating/releasing
new assets... 
\begin{enumerate}
\item If your work uses existing assets, did you cite the creators? \answerYes{We cited the authors of the codes that are included in our implementations.} 
\item Did you mention the license of the assets? \answerNo{All the assets we used are published for research purposes.}
\item Did you include any new assets either in the supplemental material
or as a URL? \answerYes{The tools used in our implementation are presented in our supplementary material.}
\item Did you discuss whether and how consent was obtained from people whose
data you're using/curating? \answerNo{The data we used are open to the public.} 
\item Did you discuss whether the data you are using/curating contains personally
identifiable information or offensive content? \answerNo{The data we used are open to the public.} 
\end{enumerate}
\item If you used crowdsourcing or conducted research with human subjects... 
\begin{enumerate}
\item Did you include the full text of instructions given to participants
and screenshots, if applicable? \answerNA{} 
\item Did you describe any potential participant risks, with links to Institutional
Review Board (IRB) approvals, if applicable? \answerNA{} 
\item Did you include the estimated hourly wage paid to participants and
the total amount spent on participant compensation? \answerNA{} 
\end{enumerate}
\end{enumerate}

\clearpage
\appendix
\begin{center}
{\bf \large Supplement to ``Stochastic Multiple Target Sampling Gradient Descent''}
\end{center}

These appendices provide supplementary details and results of MT-SGD, including our theory development and additional experiments. This consists of the following sections:
\begin{itemize}
\item Appendix \ref{sec:Proofs} contains the proofs and derivations of our theory development.
\item Appendix \ref{sec:Additional-Experiments} contains the network architectures,
experiment settings of our experiments and additional ablation studies.
\end{itemize}

\section{Proofs of Our Theory Development \label{sec:Proofs}}

\subsection{Derivations for the Taylor expansion formulation}
We have
\begin{align}
\nabla_{\epsilon}D_{KL}\left(q^{[T]}\Vert p_{i}\right)\Big|_{\epsilon=0}=-\left\langle \phi,\psi_{i}\right\rangle _{\mathcal{H}_{k}^{d}}. \label{eq:key_equation}
\end{align}

\textit{Proof of Equation~(\ref{eq:key_equation}):}
Since $T$ is assumed to be an invertible mapping, we have the following equations:
\begin{equation*}
D_{KL}\left(q^{[T]}\Vert p_{i}\right) = D_{KL}\left(T\#q\Vert p_{i}\right) = D_{KL}(q||T^{-1}\# p_{i})
\end{equation*}
and
\begin{equation}
D_{KL}(q||T^{-1}\# p_{i})=D_{KL}(q||T^{-1}\# p_{i})\big|_{\epsilon=0}+\epsilon\nabla_{\epsilon}D_{KL}(q||T^{-1}\# p_{i})\big|_{\epsilon=0}+O(\epsilon^{2}). \label{eq:kl_taylor}
\end{equation}

According to the change of variables formula, we have $T^{-1}\# p_{i}(\theta)=p_i(T(\theta))|\det\nabla_\theta T(\theta)|$,
then: 
\[
D_{KL}(q||T^{-1}\# p_{i})=\mathbb{E}_{\theta\sim q}[\log q(\theta)-\log p_{i}(T(\theta))-\log|\det\nabla_{\theta}T(\theta)|].
\]

Using this, the first term in Equation (\ref{eq:kl_taylor}) is rewritten
as: 
\begin{align}
D_{KL}(q||p_{i})= & D_{KL}(T\# q||p_{i})\big|_{\epsilon=0}=D_{KL}(q||T^{-1}\# p_{i})\big|_{\epsilon=0}\nonumber \\
= & \mathbb{E}_{\theta\sim q}[\log q(\theta)-\log p_{i}(\theta)-\log|\det\nabla_{\theta}\theta|]=\mathbb{E}_{\theta\sim q}[\log q(\theta)-\log p_{i}(\theta)]. \label{eq:kl_coef_0}
\end{align}

Similarly, the second term in Equation (\ref{eq:kl_taylor}) could
be expressed as:

\begin{align}
\nabla_{\epsilon}D_{KL}(q||T^{-1}\# p_{k})\big|_{\epsilon=0} & =\mathbb{E}_{\theta\sim q}[\nabla_{\epsilon}\log q(\theta)-\nabla_{\epsilon}\log p_{i}(T(\theta))-\nabla_{\epsilon}\log|\det\nabla_{\theta}T(\theta)|\Big]\big|_{\epsilon=0}\nonumber \\
 & =-\mathbb{E}_{\theta\sim q}[\nabla_{\epsilon}\log p_{i}(T(\theta))+\nabla_{\epsilon}\log|\det\nabla_{\theta}T(\theta)|\Big]\big|_{\epsilon=0}\nonumber \\
 & =-\mathbb{E}_{\theta\sim q}[\nabla_{T}\log p_{i}(T(\theta))\nabla_{\epsilon}T(\theta)\Big]\big|_{\epsilon=0}\nonumber \\
 & \;\;\;\;\;\;\;\;\;\;\;\;\;\;\;\;\;\;\;-\mathbb{E}_{\theta\sim q}\Big[\frac{1}{|\det\nabla_{\theta}T(\theta)|}\frac{|\det\nabla_{\theta}T(\theta)|}{\det\nabla_{\theta}T(\theta)}\nabla_{\epsilon}\det\nabla_{\theta}T(\theta)\Big]\big|_{\epsilon=0}\label{eq:kl_firstorder}\\
 & =-\mathbb{E}_{\theta\sim q}[\nabla_{T}\log p_{i}(T(\theta))\phi(\theta)]\big|_{\epsilon=0}\nonumber \\
 & \;\;\;\;\;\;\;\;\;\;\;\;\;\;\;\;\;\;\;\;-\mathbb{E}_{\theta\sim q}p\Big[\frac{\det\nabla_{\theta}T(\theta)\,\trace((\nabla_{\theta}T(\theta)^{-1}\nabla_{\epsilon}\nabla_{\theta}T(\theta))}{\det\nabla_{\theta}T(\theta\theta)}\Big]\big|_{\epsilon=0}\nonumber \\
 & =-\mathbb{E}_{\theta\sim q}[\nabla_{\theta}\log p_{i}(\theta)\phi(\theta)+\trace(\nabla_{\theta}\phi(\theta))]. \nonumber 
\end{align}

It could be shown from the reproducing property of the RKHS that $\phi_{i}(\theta)=\left\langle \phi_{i}(\cdot),k(\theta,\cdot)\right\rangle _{\mathcal{H}_{k}}$,
then we find that 
\begin{equation}
\frac{\partial\phi_{i}(\theta)}{\partial\hat{\theta}_{i}}=\left\langle \phi_{i}(\cdot),\frac{\partial k(\theta,\cdot)}{\partial\hat{\theta}_{i}}\right\rangle _{\mathcal{H}_{k}}. \label{eq:derivative_RKHS}
\end{equation}

Let $U_{d\times d}=\nabla_{\theta}\phi(\theta)$ whose $u_{i}^{T}$ denotes
the $i^{th}$ row vector and the particle $\theta\in \mathbb{R}^{d}$ is represented by $\{\hat{\theta}\}_{i=1}^{d}$, the row vector $u_{i}^{T}$ is given by: 
\begin{equation}
u_{i}^{T}\coloneqq\frac{\partial\phi_{i}(\theta)}{\partial \theta}=\frac{\partial\phi_{i}(\theta)}{\partial(\hat{\theta}_{1},\hat{\theta}_{2},\dots,\hat{\theta}_{d})}=\Big[\frac{\partial\phi_{i}(\theta)}{\partial\hat{\theta}_{1}};\frac{\partial\phi_{i}(\theta)}{\partial\hat{\theta}_{2}};\dots;\frac{\partial\phi_{i}(\theta)}{\partial\hat{\theta}_{d}}\Big]. \label{eq:derivative_particles}
\end{equation}

Combining Property (\ref{eq:derivative_RKHS}) and Equation (\ref{eq:derivative_particles}),
we have: 
\begin{align}
u_{i}^{T} & \coloneqq\Big[\frac{\partial\phi_{i}(\theta)}{\partial\hat{\theta}_{1}};\frac{\partial\phi_{i}(\theta)}{\partial\hat{\theta}_{2}};\dots;\frac{\partial\phi_{i}(\theta)}{\partial\hat{\theta}_{d}}\Big]\nonumber \\
 & =\Big[\left\langle \phi_{i}(\cdot),\frac{\partial k(\theta,\cdot)}{\partial\hat{\theta}_{1}}\right\rangle _{\mathcal{H}_{k}};\left\langle \phi_{i}(\cdot),\frac{\partial k(\theta,\cdot)}{\partial\hat{\theta}_{2}}\right\rangle _{\mathcal{H}_{k}};\dots;\left\langle \phi_{i}(\cdot),\frac{\partial k(\theta,\cdot)}{\partial\hat{\theta}_{d}}\right\rangle _{\mathcal{H}_{k}}\Big]. \label{eq:derivative_particles2}
\end{align}

Substituting Equation~(\ref{eq:derivative_particles2}) to Equation~(\ref{eq:kl_firstorder}),
the linear term of the Taylor expansion could be derived as: 
\begin{align*}
\nabla_{\epsilon}D_{KL}(q||T^{-1}\# p_{i})\big|_{\epsilon=0} 
& =-\mathbb{E}_{\theta\sim q}[\nabla_{\theta}\log p_{i}(\theta)\phi(\theta)+\trace(\nabla_{\theta}\phi(\theta))]\nonumber   \\
 & =-\mathbb{E}_{\theta\sim q}\Big[\sum_{j=1}^{d}\left\langle \phi_{j}(\cdot),k(\theta,\cdot)\right\rangle _{\mathcal{H}_{k}}(\nabla_{\theta}\log p_{i}(\theta))_{j}+\frac{\partial\phi_{j}(\theta)}{\partial\hat{\theta}_{j}})\Big]\nonumber \\
 & =-\sum_{j=1}^{d}\mathbb{E}_{\theta\sim q}\Big[\left\langle \phi_{j}(\cdot),k(\theta,\cdot)(\nabla_{\theta}\log p_{i}(\theta))_{j}\right\rangle _{\mathcal{H}_{k}}+\left\langle \phi_{j}(\cdot),\Big(\frac{\partial k(\theta,\cdot)}{\partial \theta}\Big)_j\right\rangle _{\mathcal{H}_{k}}\Big]\label{eq:kl_coef_1}\\
 & =-\sum_{j=1}^{d}\left\langle \phi_{j}(\cdot),\mathbb{E}_{\theta\sim q}\Big[k(\theta,\cdot)(\nabla_{\theta}\log p_{i}(\theta))_{j}+\Big(\frac{\partial k(\theta,\cdot)}{\partial \theta}\Big)_j\Big]\right\rangle _{\mathcal{H}_{k}}\nonumber \\
 & =-\left\langle \phi(\cdot),\psi(\cdot)\right\rangle _{\mathcal{H}_{k}^{d}}, \nonumber 
\end{align*}
where $(v)_j$ denotes the $j$-th element of $v$ and   $\psi(\cdot)\in\mathcal{H}_{k}^{d}$ is a matrix whose $j^{th}$
column vector is given by 
\[
\mathbb{E}_{\theta\sim q}\Big[k(\theta,\cdot)(\nabla_{\theta}\log p_{i}(\theta))_{j}+\Big(\frac{\partial k(\theta,\cdot)}{\partial \theta}\Big)_j\Big].
\]
In other word, the formula of $\psi(\cdot)$ becomes
\[
\mathbb{E}_{\theta\sim q}\Big[k(\theta,\cdot)\nabla_{\theta}\log p_{i}(\theta)+\frac{\partial k(\theta,\cdot)}{\partial \theta}\Big].
\]
As a consequence, we obtain the conclusion of Equation~(\ref{eq:key_equation}). 

\subsection{Proof of Lemma \ref{lem:common_direction} \label{subsec:common_direction}}
Before proving this lemma, let us re-state it:
\begin{lem}
Let $w^{*}$ be the optimal solution of the optimization problem  $w^{*}=\argmin{w\in\simplex_{K}}w^{T}Uw$
and $\phi^{*}=\sum_{i=1}^{K}w_{i}^{*}\phi_{i}^{*}$, where $\simplex_{K}=\left\{ \pi\in\mathbb{R}_{+}^{K}:\norm{\pi}_{1}=1\right\} $
and $U\in\mathbb{R}^{K\times K}$ with $U_{ij}=\left\langle \phi_{i}^{*},\phi_{j}^{*}\right\rangle _{\mathcal{H}_{k}^{d}}$,
then we have
\[
\left\langle \phi^{*},\phi_{i}^{*}\right\rangle _{\mathcal{H}_{k}^{d}}\geq\norm{\phi^{*}}_{\mathcal{H}_{k}^{d}}^{2},i=1,...,K. 
\]
\end{lem}

\textit{Proof.} For arbitrary ${\epsilon\in[0,1]}$ and $u\in\simplex_{K}$, then $ \omega := \epsilon u + (1-\epsilon)w^{*} \in \simplex_{K}$, we thus have the following inequality:
\begin{align*}
 \hspace{1cm}{w^{*}}^{T}Uw^{*} & \leq \omega^TU\omega \nonumber\\
 & = {(\epsilon u + (1-\epsilon)w^{*})}^T U (\epsilon u + (1-\epsilon)w^{*}) \\
  & =  {(w^*+\epsilon(u-w^*)}^TU{(w^*+\epsilon(u-w^*)} \\
  & = {w^{*}}^{T}Uw^{*} + 2\epsilon {w^*}^TU(u-{w^*}) + \epsilon^2 (u-{w^*})^T U (u-{w^*}),
\end{align*}
which is equivalent to
\begin{equation}
  0 \leq  2\epsilon {w^*}^TU(u-{w^*}) + \epsilon^2 (u-{w^*})^T U (u-{w^*}).   \label{eq:quadratic}
\end{equation}
Hence ${w^*}^TU(u-{w^*})\geq0$, since otherwise the R.H.S of inequality~(\ref{eq:quadratic}) will be negative with  sufficiently small $\epsilon$. By that, we arrive at
\begin{align*}
{w^*}^TU{w^*} \leq {w^*}^TUu.
\end{align*}
By choosing $u$ to be a one hot vector at $i$, we obtain the conclusion of Lemma~\ref{lem:common_direction}.

\subsection{Derivations for the matrix $U_{ij}$'s formulation in Equation~(\ref{eq:Uij})}

We have
\begin{align*}
\phi_{i}^{*}\left(\cdot\right) & = \mathbb{E}_{\theta\sim q}\left[k\left(\theta,\cdot\right)\nabla\log p_{i}\left(\theta\right)+\nabla k\left(\theta,\cdot\right)\right], \\
\phi_{j}^{*}\left(\cdot\right) & = \mathbb{E}_{\theta'\sim q}\left[k\left(\theta',\cdot\right)\nabla\log p_{j}\left(\theta'\right)+\nabla k\left(\theta',\cdot\right)\right].
\end{align*}
Therefore, we find that
\begin{align*}
U_{ij} & =\left\langle \phi_{i}^{*},\phi_{j}^{*}\right\rangle _{\mathcal{H}_{k}^{d}}=\mathbb{E}_{\theta,\theta'\sim q}\Biggl[\left\langle k\left(\theta,\cdot\right),k\left(\theta',\cdot\right)\right\rangle _{\mathcal{H}_{k}}\sum_{l=1}^{d}\nabla_{\theta_{l}}\log p_{i}(\theta)\nabla_{\theta_{l}'}\log p_{j}(\theta')\\
 & +\sum_{l=1}^{d}\nabla_{\theta_{l}}\log p_{i}(\theta)\left\langle k\left(\theta,\cdot\right),\nabla_{\theta_{l}'}k\left(\theta',\cdot\right)\right\rangle _{\mathcal{H}_{k}}+\sum_{l=1}^{d}\nabla_{\theta_{l}'}\log p_{j}(\theta')\left\langle k\left(\theta',\cdot\right),\nabla_{\theta_{l}}k\left(\theta,\cdot\right)\right\rangle _{\mathcal{H}_{k}}\\
 & +\sum_{l=1}^{d}\left\langle \nabla_{\theta_{l}}k\left(\theta,\cdot\right),\nabla_{\theta_{l}'}k\left(\theta',\cdot\right)\right\rangle _{\mathcal{H}_{k}}\Biggl],
\end{align*}
which is equivalent to
\begin{align*}
U_{ij} & =\mathbb{E}_{\theta,\theta'\sim q}\Biggl[k\left(\theta,\theta'\right)\left\langle \nabla\log p_{i}(\theta),\nabla\log p_{j}(\theta')\right\rangle \\
 & +\left\langle \nabla\log p_{i}(\theta),\frac{\partial k(\theta,\theta')}{\partial\theta'}\right\rangle +\left\langle \nabla\log p_{j}(\theta'),\frac{\partial k(\theta,\theta')}{\partial\theta}\right\rangle +\\
 & +\sum_{l=1}^{d}\left\langle \nabla_{\theta_{l}}k\left(\theta,\cdot\right),\nabla_{\theta_{l}'}k\left(\theta',\cdot\right)\right\rangle _{\mathcal{H}_{k}}\Biggl].
\end{align*}
Now, note that 
\[
\left\langle k\left(\theta,.\right),\varphi\left(.\right)\right\rangle _{\mathcal{H}_{k}}=\varphi\left(\theta\right),
\]
hence we gain
\[
\left\langle \nabla_{\theta_{l}}k\left(\theta,.\right),\varphi\left(.\right)\right\rangle _{\mathcal{H}_{k}}=\nabla_{\theta_{l}}\varphi\left(\theta\right),
\]
which follows that 
\[
\left\langle \nabla_{\theta_{l}}k\left(\theta,\cdot\right),\nabla_{\theta_{l}'}k\left(\theta',\cdot\right)\right\rangle _{\mathcal{H}_{k}}=\nabla_{\theta_{l},\theta'_{l}}^{2}k\left(\theta,\theta'\right),
\]
\[
\sum_{l=1}^{d}\left\langle \nabla_{\theta_{l}}k\left(\theta,\cdot\right),\nabla_{\theta_{l}'}k\left(\theta',\cdot\right)\right\rangle _{\mathcal{H}_{k}}=\sum_{l=1}^{d}\nabla_{\theta_{l},\theta'_{l}}^{2}k\left(\theta,\theta'\right)=\trace\left(\frac{\partial^{2}k(\theta,\theta')}{\partial\theta\partial\theta'}\right).
\]
Putting these results together, we obtain that
\begin{align*}
U_{ij} & =\left\langle \phi_{i}^{*},\phi_{j}^{*}\right\rangle _{\mathcal{H}_{k}^{d}}=\mathbb{E}_{\theta,\theta'\sim q}\Biggl[k\left(\theta,\theta'\right)\left\langle \nabla\log p_{i}(\theta),\nabla\log p_{j}(\theta')\right\rangle \nonumber \\
 & +\left\langle \nabla\log p_{i}(\theta),\frac{\partial k(\theta,\theta')}{\partial\theta'}\right\rangle +\left\langle \nabla\log p_{j}(\theta'),\frac{\partial k(\theta,\theta')}{\partial\theta}\right\rangle +\trace\left(\frac{\partial^{2}k(\theta,\theta')}{\partial\theta\partial\theta'}\right)\Biggr].
\end{align*}
As a consequence, we obtain the conclusion of Equation~(\ref{eq:Uij}).

\subsection{Proof of Theorem \ref{thm:KL_reduce} \label{subsec:KL_reduce}}
Before proving this theorem, let us re-state it:

\begin{thm}
 $w\in\simplex_{K}$
such that $\sum_{i=1}^{K}w_{i}\phi_{i}^{*}=0$, given a sufficiently
small step size $\epsilon$, all KL divergences w.r.t. the target
distributions are strictly decreased by at least $A\norm{\phi^{*}}_{\mathcal{H}_{k}^{d}}^{2}>0$
where $A$ is a positive constant.
\end{thm}
\begin{proof}
We have for all $i=1,...,K$ that
\begin{align*}
D_{KL}\left(q^{[T]}\Vert p_{i}\right) & =D_{KL}\left(q\Vert p_{i}\right)+\nabla_{\epsilon}D_{KL}\left(q^{[T]}\Vert p_{i}\right)\Big|_{\epsilon=0}\epsilon+O_{i}\left(\epsilon^{2}\right)\\
 & =D_{KL}\left(q\Vert p_{i}\right)-\left\langle \phi,\psi_{i}\right\rangle _{\mathcal{H}_{k}^{d}}\epsilon+O_{i}\left(\epsilon^{2}\right)\\
 & \leq D_{KL}\left(q\Vert p_{i}\right)-\norm{\phi^{*}}_{\mathcal{H}_{k}^{d}}^{2}\epsilon+O_{i}\left(\epsilon^{2}\right).
\end{align*}

Because $\lim_{\epsilon\goto0}\frac{O_{i}(\epsilon^{2})}{\epsilon^{2}}=B_{i}$,
there exists $\alpha_{i}>0$ such that $|\epsilon|<\alpha_{i}$ implies
$\left|O_{i}(\epsilon^{2})\right|<\frac{3}{2}\left|B_{i}\right|\epsilon^{2}$.
By choosing, $B=\frac{3}{2}\max_{i}\text{\ensuremath{\left|B_{i}\right|}}$
and $\alpha=\min_{i}\alpha_{i}$, we arrive at for all $\epsilon<\alpha$
and all $i$
\[
D_{KL}\left(q^{[T]}\Vert p_{i}\right)<D_{KL}\left(q\Vert p_{i}\right)-\norm{\phi^{*}}_{\mathcal{H}_{k}^{d}}^{2}\epsilon+B\epsilon^{2}.
\]

Finally, by choosing sufficiently small $\epsilon>0$, we reach the
conclusion of the theorem.
\end{proof} 



\section{Implementation Details \label{sec:Additional-Experiments}}
In this appendix, we provide implementation details regarding the empirical evaluation in the main paper along with additional comparison experiments.
\subsection{Experiments on Toy Datasets}
\subsubsection{Sampling from Multiple Distribution}
In this experiment, the three target distributions are created as presented in the main paper. The particle's coordinates are randomly sampled from the normal distribution $\mathcal{N}\left(0, 5\right)$. Adam optimizer \cite{kingma2014adam} with learning rate of $3e-2$ and $\beta_1 = 0.9, \beta_2 = 0.999$ is used to update the particles. MOO-SVGD and MT-SGD converged after 2000 and 1000 iterations, respectively.

\begin{figure}[!ht]
\begin{centering}

\includegraphics[width=1\textwidth]{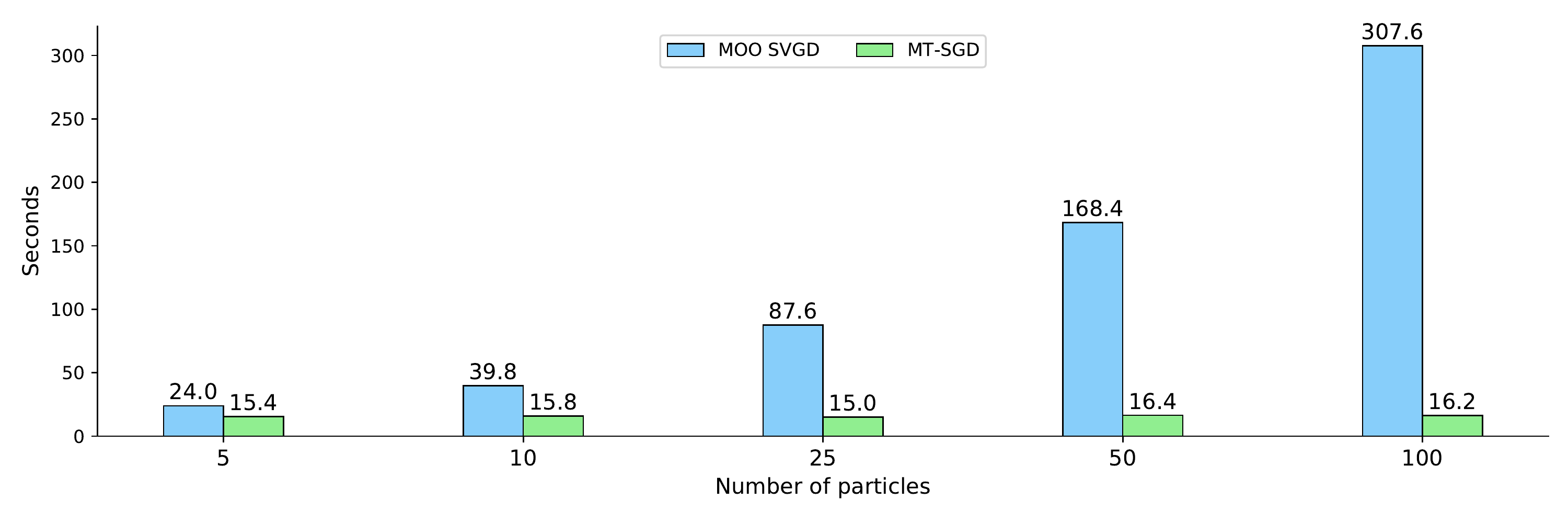}
\par\end{centering}
\caption{Running time of MT-SGD and MOO-SVGD for 1000 steps on: Intel(R) Xeon(R) CPU @ 2.20GHz CPU and
Tesla T4 16GB VRAM GPU. Results are averaged over 5 runs.}
\label{fig:runtime}

\end{figure}

We also measure the running time between MOO-SVGD and our proposed method when varying the number of particles from $5$ to $100$. In Figure \ref{fig:runtime}, we plot the time consumption when running MOO-SVGD and MT-SGD in 1000 iterations. As can be seen that, MOO-SVGD runtime grows linearly with the number of particles, since it requires solving separate quadratic problems (Algorithm \ref{alg:alg}) for each particle. By contrast, there is only one quadratic programming problem solving in our proposed method, which significantly reduces time complexity, especially when the number of particles is high.

\subsubsection{Multi-objective Optimization}

ZDT-3 \cite{zitzler2000comparison} is a classic benchmark problem in multi-objective optimization with 30 variables $\theta = (\theta_1, \theta_2,\dots, \theta_{30})$ with a number of disconnected Pareto-optimal fronts. This problem is given by:
\begin{align*}
    & \min f_1(\theta),  \\
    & \min f_2(\theta) = g(\theta) h(f_1(\theta), g(\theta)),
\end{align*}
where
\begin{align*}
&f_1(\theta) = \theta_1, \\
& g(\theta) = 1 + \frac{9}{29} \sum_{i=2}^{30} \theta_i, \\
& h(f_1, g) = 1 - \sqrt{\frac{f_1}{g}} - \frac{f_1}{g}\sin (10\pi f_1),\\
& 0 \leq \theta_i \leq 1 , i = 1, 2, \dots, 30.
\end{align*}
The Pareto optimal solutions are given by
\begin{align*}
   & 0 \leq \theta_1 \leq 0.0830, \\
   & 0.1822 \leq \theta_1 \leq 0.257,\\
   & 0.4093 \leq \theta_1 \leq 0.4538,\\
   & 0.6183 \leq \theta_1 \leq 0.6525,\\
   & 0.8233 \leq \theta_1 \leq 0.8518,\\
  &  \theta_i=0 \text{ for } i= 2,\dots, 30.
\end{align*}

For ZDT3 experiment, we utilize Adam \cite{kingma2014adam}, learning rate $5e-4$ and update the 50 particles for 10000 iterations as in the comparative baseline \cite{liu2021profiling}.

\subsubsection{Multivariate regression}
We consider the SARCOS regression dataset \cite{vijayakumar2002statistical}, which contains 44,484 training samples and 4,449 testing samples with 21 input variables and 7 outputs (tasks). The train-test split in \cite{navon2021learning} is kept, with 40,036 training
examples, 4,448 validation examples, and 4,449 test examples. We replicate the neural network architecture from \cite{navon2021learning} as follows:
$
 \mathrm{21\times256\,FC \shortrightarrow ReLU}
     \shortrightarrow \mathrm{256\times256\,FC \shortrightarrow ReLU} \shortrightarrow \mathrm{256\times256\,FC \shortrightarrow ReLU} \shortrightarrow \mathrm{256\times7\,FC }
$ (139,015 params). The network is optimized by Adam \cite{kingma2014adam} optimizer for 1000 epochs, with $\beta_1 = 0.9, \beta_2 = 0.999$, and the learning rate of $1e-4$. All experimental results are obtained by running five times with different seeds.

\begin{table}[H]

\caption{Mean square errors of MT-SGD and competing methods on SARCOS dataset  \cite{vijayakumar2002statistical}. We take the best checkpoint in each approach based on the validation score. Results are averaged over 5 runs, and we highlight the best method for each task in \textbf{bold}. \label{tab:sarcos}}
\vspace{1mm}
\centering{}\resizebox{1.0\textwidth}{!}{
\begin{tabular}{c|c|l|l|l|l|l|l|l|l}

\toprule
 & Method & Task 1 & Task 2 & Task 3 & Task 4 & Task 5 & Task 6 & Task 7  & Average\\
\midrule
\multirow{3}{*}{ Validation} &  MGDA & 0.025 &  0.2789 & 0.0169 & 0.0026 & 1.158 & 0.264 & 0.005 & 0.25 \\
 & MOO-SVGD & 0.0177 & 0.2182 & 0.0113 & 0.0013 & 1.241 & 0.2292 & 0.0025 & 0.2459 \\
 & MT-SGD &  \textbf{0.0173} & \textbf{0.2124} & \textbf{0.0112} & \textbf{0.0012} & \textbf{1.110} & \textbf{0.2208} & \textbf{0.0024} & \textbf{0.2251} \\
\midrule
\multirow{3}{*}{ Test} &  MGDA & 0.0082 & 0.0675 & 0.0038 & 0.0009 & 0.2635 & 0.0455 & 0.0018 & 0.0559 \\
 & MOO-SVGD & 0.0043 & 0.0586 & 0.0019 & 0.0003 & 0.2584 & 0.0365 & 0.0007 & 0.0515 \\
 & MT-SGD & \textbf{0.0037} & \textbf{0.0515} & \textbf{0.0018} & \textbf{0.0002} & \textbf{0.2097} & \textbf{0.0318} & \textbf{0.0005} & \textbf{0.0428} \\
\bottomrule
\end{tabular}}
\end{table}

Regarding the baselines for this experiment, we compare our method against MGDA \cite{sener2018multi}, MOO-SVGD \cite{liu2021profiling}. We empirically set the batch size as $512$ and $M=5$ particles in MT-SGD and competing methods. The mean square error for each task and the average results are shown in Table \ref{tab:sarcos}. We find that our method achieves the lowest error on all tasks, with the largest gap on Task 4. MT-SGD outperforms the second-best method, MOO-SVGD, with $0.2251$ vs. $0.2459$ on validation set and $0.0428$  vs. $0.0515$ on test set.

\subsection{Experiments on Real Datasets}

\subsubsection{ Experiments on Multi-Fashion+Multi-MNIST Datasets}

We follow the same training protocol with previous work \cite{lin2019pareto, liu2021profiling, sener2018multi}, Lenet \cite{mnist} is trained in 100 epochs with SGD optimizer. The input images are in size $36 \times 36$ and the training batch size is $256$.

\begin{figure}[!ht]
\begin{centering}

\includegraphics[width=.8\textwidth]{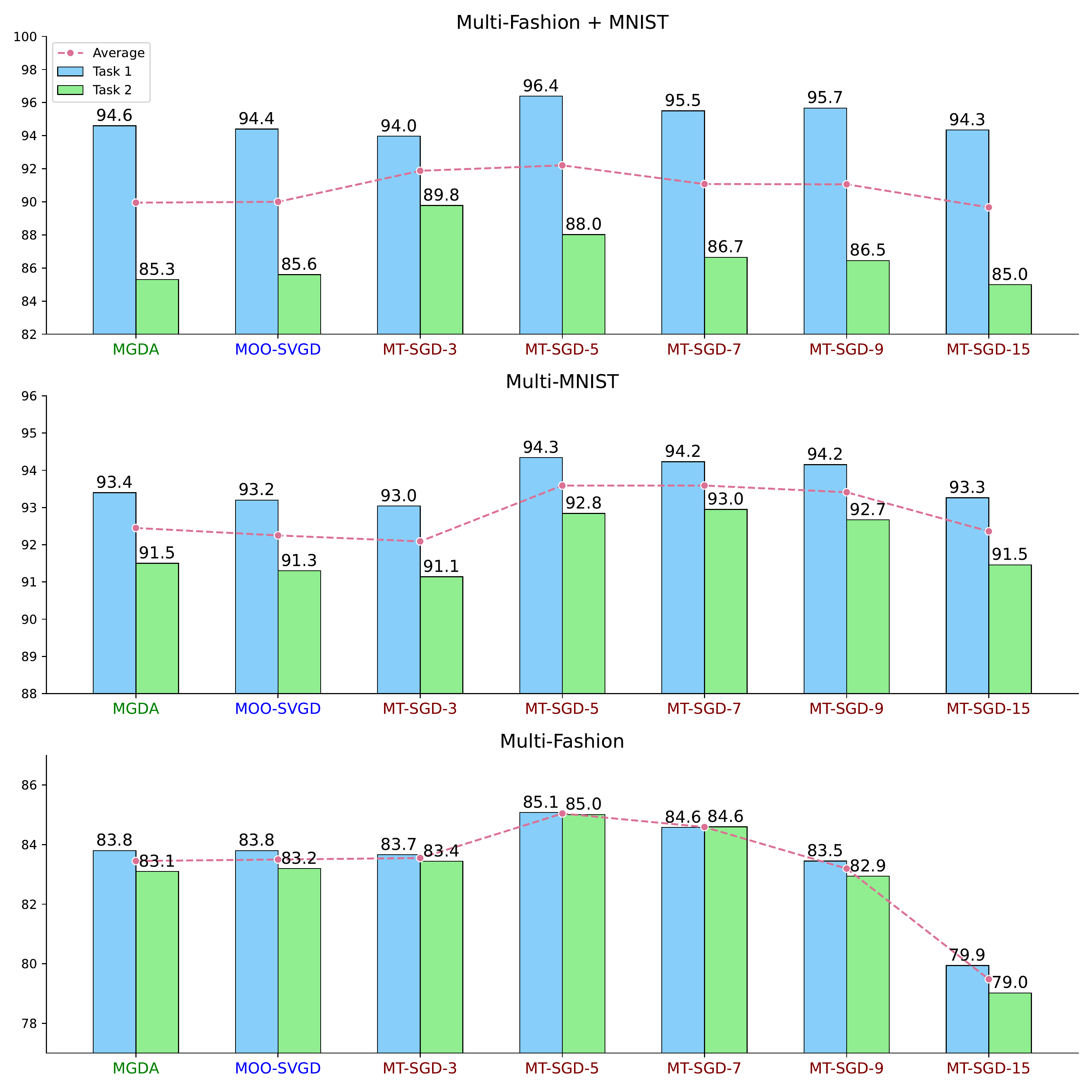}
\par\end{centering}
\caption{Average accuracy (\%) when varying the number of particles from $3$ to $15$. \textcolor{Maroon}{MT-SGD-m} denotes our method using m particle networks.}
\label{fig:vary_particle}

\end{figure}

We now study the performance of our proposed method against variation in the number of particles by conducting more experiments on Multi-MNIST/Fashion/Fashion+MNIST datasets. We vary the number of neural networks in $3,5,7,9, 15$ and present the accuracy scores in Figure \ref{fig:vary_particle}. From the results, a simple conclusion that can be derived is that increasing the number of particle networks from $3 \shortrightarrow 5$ improves the performance in all three datasets, surpassing the other two baselines, while further increasing this hyperparameter does not help.

\textbf{Metrics:} In the main paper, we compare our proposed method against baselines in terms of Brier score and expected calibration error (ECE). We here provide more details on how to calculate these metrics. Assumed that the training dataset $\mathcal{D}$ consists of $N$ i.i.d examples $\mathcal{D} = \{x_n, y_n\}_{n=1}^N$ where $y_n \in \{1,2,\dots,K\}$ denotes corresponding labels for K-class classification problem. Let $p(y = c | x_i)$ be the predicted confidence that $x_i$ belongs to class C.
\begin{itemize}
    \item \textbf{Brier score:} The Brier score is computed as the squared error between a predicted probability $p(y|x_i)$ and the one-hot vector ground truth:
    \[BS = \frac{1}{N} \sum_{i=1}^N \sum_{c=1}^K \Big( \mathbf{1}_{y_i=c} - p(y=c|x_i)  \Big)^2\]
    \item \textbf{Expected calibration error:}  Partitioning predictions into $M$ equally-spaced bins $B_m = \Big(\frac{m-1}{M}, \frac{m}{M}\Big] (m=1,2\dots, M)$, the expected calibration error is computed as the average gap between the accuracy and the predicted confidence within each bin:
    \[\mathrm{ECE}=\sum_{m=1}^M \frac{|B_m|}{N}|\mathrm{acc}(B_m)-\mathrm{conf}(B_m)|\]
    where $\mathrm{acc}(B_m)$ and $\mathrm{conf}(B_m)$ denote the accuracy and confidence of bin $B_m$
\end{itemize}

\begin{figure}[!ht]
\begin{centering}

\includegraphics[width=1.\textwidth]{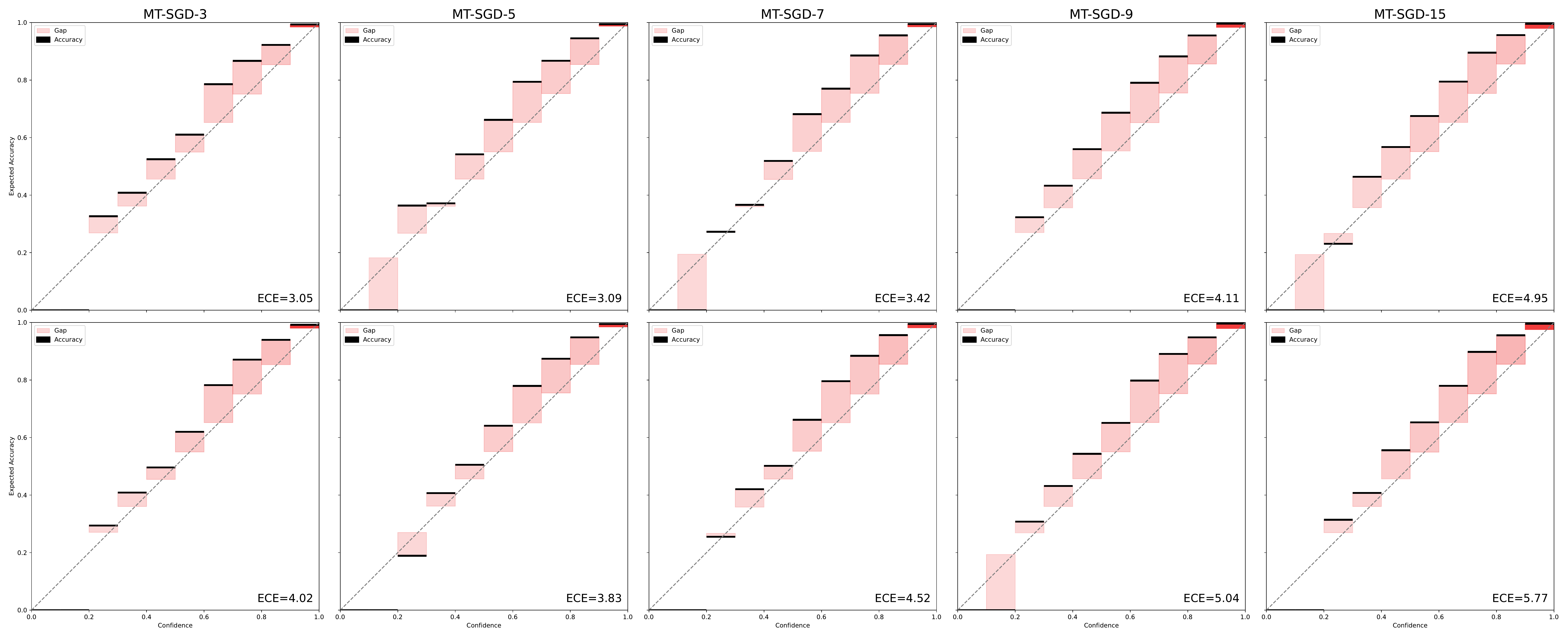}
\par\end{centering}
\caption{Expected Calibration Error (\%) when varying the number of particles from $3$ to $15$ on Multi-MNIST. {MT-SGD-m} denotes our method using m particle networks. We set the number of bins equal to $10$ throughout the experiments.}
\label{fig:ece_particle}

\end{figure}

Figure \ref{fig:ece_particle} displays ECE as a function of the number of particles. Similar to the average accuracy metric, the Expected Calibration Error reduces when we increase the number of particle networks from $3$ to $5$ yet does not decrease in the cases of MT-SGD-7, MT-SGD-9 and MT-SGD-15.

{\textbf{Computational complexity of MT-SGD:}  From the complexity point of view, MT-SGD introduces a marginal computational overhead compared to MGDA since it requires calculating the matrix $U$,  which has a  complexity $O(K^2M^2d)$, where the number of particles $M$  is usually set to a small positive integer. However, on the one hand, computing $U$'s entries can be accelerated in practice by calculating them in parallel since there is no interaction between them during forward pass. On the other hand, the computation of the back-propagation is typically more costly than the forward pass. Thus, the main bottlenecks in our method lie on the backward pass and solving the quadratic programming problem - which is an iterative method \cite{jaggi2013revisiting, sener2018multi}.}

\begin{figure}[!ht]
\begin{centering}

\includegraphics[width=1\textwidth]{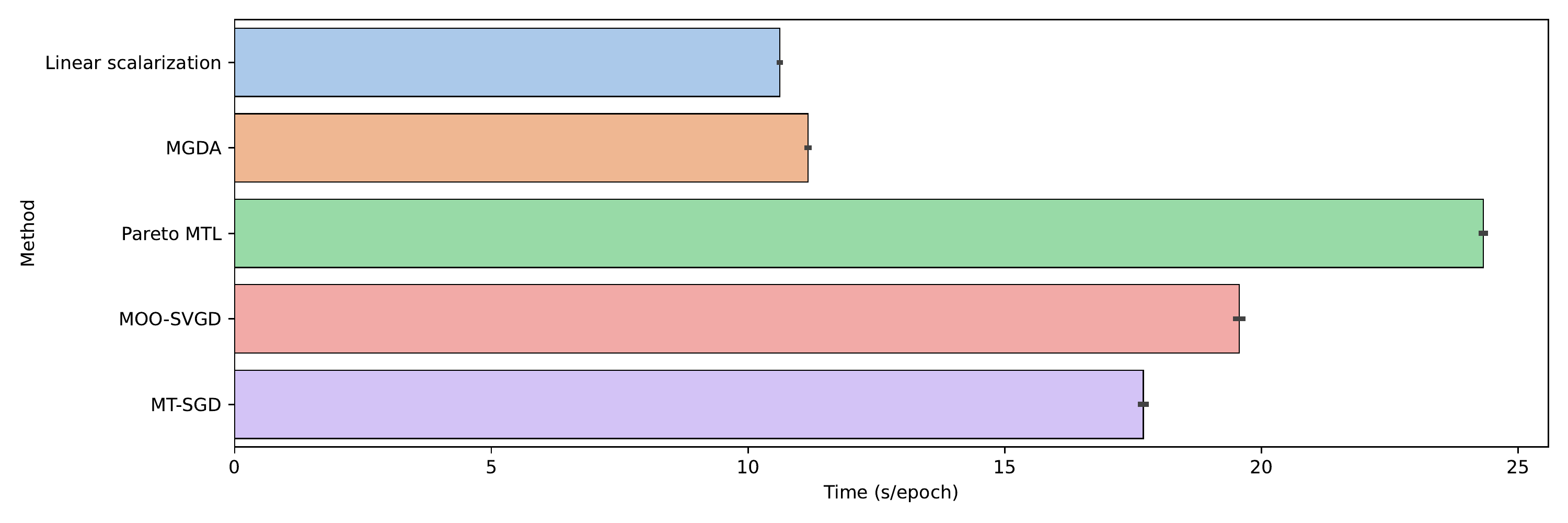}
\par\end{centering}
\caption{Running time on each epoch of MT-SGD and compared baselines on Multi-MNIST dataset. Results are averaged over 5 runs, with the standard deviation reported by error bars.}
\label{fig:runtime_method}

\end{figure}

{As a final remark in the  Multi-Fashion+Multi-MNIST experiment, we compare our methods against baselines in terms of the required  running time in a single epoch and plot the result in Figure \ref{fig:runtime_method}. We observe that in our experiments on Multi-MNIST dataset, the computation time of methods that enforce the diversity of obtained models is higher than that of the methods that do not (Pareto MTL, MOO-SVGD, MT-SGD vs Linear scalarization, MGDA). Nevertheless, this is a small price to pay for the major gain in the ensemble performance, since our training involves the interaction between models to impose diversity. Compared to Pareto MTL and MOO-SVGD, the running time of our proposed MT-SGD is considerably better (~2s less than MOO-SVGD and ~6s less than Pareto MTL).}

\subsubsection{Experiment on CelebA Dataset}

We performed our experiments on the CelebA dataset \cite{liu2015deep}, which contains images annotated with 40 binary attributes. Resnet-18 backbone \cite{he2016deep} without the final layer as a shared encoder and a 2048 x 2 dimensional fully connected layer for each task is employed as in \cite{sener2018multi}. We train this network for 100 epochs with Adam optimizer \cite{kingma2014adam} of learning rate $5e-4$ and batch size $64$. All images are resized to $64\times64\times3$.

Due to space constraints, we report only the abbreviation of each task in the main paper, their full names are presented below.

\begin{table}[!ht]

\caption{CelebA binary classification tasks full names.\label{tab:celeba-name}}
\vspace{1mm}
\centering{}\resizebox{1.0\textwidth}{!}{
\begin{tabular}{c|c|c|c|c|c|c|c|c|c}
\toprule
  5S & AE & Att & BUE & Bald & Bangs & BL & BN & BlaH & BloH\\
\midrule
  5 O'clock Shadow & Arched Eyebrows & Attractive & Bags Under Eyes & Bald & Bangs & Big Lips & Big Nose & Black Hair & Blond Hair \\
\end{tabular}}
\end{table}

{Now we investigate the effectiveness of
our proposed MT-SGD method on the whole CelebA dataset, compared with prior work: Uniform scaling: minimizing the uniformly weighted sum of objective functions, Single task: train separate models individually for each task, Uncertainty \cite{kendall2018multi}:  adaptive reweighting with balanced uncertainty, Gradnorm \cite{chen2018gradnorm}: balance the loss functions via gradient magnitude and MGDA \cite{sener2018multi}. The results from previous work are reported in \cite{chen2018gradnorm}. For MGDA, we use their officially released codebase at \url{https://github.com/isl-org/MultiObjectiveOptimization}. For a fair comparison, we run the code with five different random seeds and present the obtained scores in Figure \ref{fig:radar-celeba}.}

{Following \cite{sener2018multi}, we divide 40 target binary attributes into two subgroups: hard and easy tasks for easier visualization. As can be seen from Figure \ref{fig:radar-celeba}, we observe that the naively trained 
Uniform scaling has relatively low performance on many tasks, e.g. ``Mustache", ``Big Lips", ``Oval Face". Compared to other baselines, our proposed method significantly reduces the prediction error in almost all the tasks, especially on ``Goatee", ``Double Chi" and ``No Beard". The detailed result for each target attribute can be found in Table \ref{tab:celeba-full}. }

\begin{figure}[!ht]
\begin{centering}

\includegraphics[width=1.\textwidth]{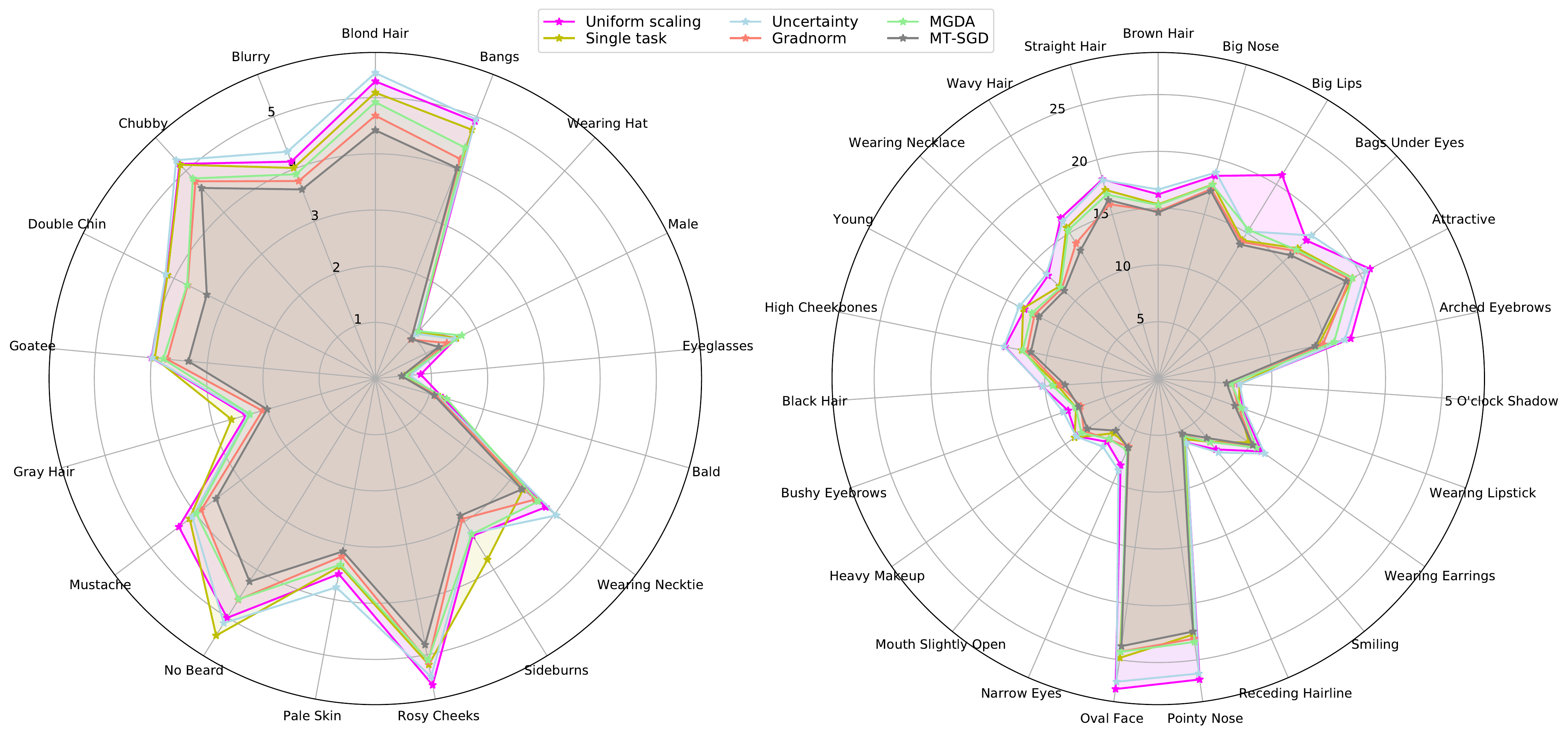}
\par\end{centering}
\caption{Radar charts of prediction error on CelebA \cite{liu2015deep} for each individual binary classification task. 
We divide attributes into two sets: easy tasks on the left, difficult tasks on the right, as in \cite{sener2018multi}.}
\label{fig:radar-celeba}

\end{figure}

\begin{table}[!ht]
\centering
\caption{Average performance (lower is better) of each target attribute for all baselines. We use the \textbf{bold} font to highlight the best-obtained score in each task. \label{tab:celeba-full}}
\resizebox{.95\columnwidth}{!}{
\begin{tabular}{l|r|r|r|r|r|r}
\midrule
Attribute & Uniform scaling & Single task & Uncertainty & Gradnorm & MGDA & MT-SGD\\
\midrule
5 O'clock Shadow & 7.11 & 7.16 & 7.18 & 6.54 & 6.47 & \textbf{6.03}\\
Arched Eyebrows & 17.30 & {14.38} & 16.77 & 14.80 & 15.80 & \textbf{14.11}\\
Attractive & 20.99 & 19.25 & 20.56 & {18.97} & 19.21 & \textbf{18.62}\\
Bags Under Eyes & 17.82 & 16.79 & 18.45 & 16.47 & 16.60 & \textbf{15.91}\\
Bald & 1.25 & 1.20 & 1.17 & 1.13 & 1.32 & \textbf{1.09}\\
Bangs & 4.91 & 4.75 & 4.95 & 4.19 & 4.41 & \textbf{4.02}\\
Big Lips & 20.97 & 14.24 & 15.17 & {14.07} & 15.32 & \textbf{13.82}\\
Big Nose & 18.53 & 17.74 & 18.84 & 17.33 & 17.70 & \textbf{17.14}\\
Black Hair & 10.22 & 8.87 & 10.19 & 8.67 & 9.31 & \textbf{8.22}\\
Blond Hair & 5.29 & 5.09 & 5.44 & 4.68 & 4.92 & \textbf{4.42}\\
Blurry & 4.14 & 4.02 & 4.33 & 3.77 & 3.90 & \textbf{3.61}\\
Brown Hair & 16.22 & 15.34 & 16.64 & 14.73 & 15.27 & \textbf{14.63}\\
Bushy Eyebrows & 8.42 & 7.68 & 8.85 & \textbf{7.23} & 7.69 & 7.42\\
Chubby & 5.17 & 5.15 & 5.26 & 4.75 & 4.82 & \textbf{4.59}\\
Double Chin & 4.14 & 4.13 & 4.17 & 3.73 & 3.74 & \textbf{3.35}\\
Eyeglasses & 0.81 & 0.52 & 0.62 & 0.56 & 0.54 & \textbf{0.47}\\
Goatee & 4.00 & 3.94 & 3.99 & 3.72 & 3.79 & \textbf{3.34}\\
Gray Hair & 2.39 & 2.66 & 2.35 & 2.09 & 2.32 & \textbf{2.00}\\
Heavy Makeup & 8.79 & 9.01 & 8.84 & 8.00 & 8.29 & \textbf{7.65}\\
High Cheekbones & 13.78 & 12.27 & 13.86 & 11.79 & 12.18 & \textbf{11.45}\\
Male & 1.61 & 1.61 & 1.58 & 1.42 & 1.72 & \textbf{1.26}\\
Mouth Slightly Open & 7.18 & 6.20 & 7.73 & 6.91 & 6.86 & \textbf{5.91}\\
Mustache & 4.38 & 4.14 & 4.08 & 3.88 & 3.99 & \textbf{3.55}\\
Narrow Eyes & 8.32 & 6.57 & 8.80 & \textbf{6.54} & 6.88 & 6.64\\
No Beard & 5.01 & 5.38 & 5.12 & 4.63 & 4.62 & \textbf{4.25}\\
Oval Face & 27.59 & 24.82 & 26.94 & 24.26 & 24.28 & \textbf{23.78}\\
Pale Skin & 3.54 & 3.40 & 3.78 & 3.22 & 3.37 & \textbf{3.13}\\
Pointy Nose & 26.74 & {22.74} & 26.21 & 23.12 & 23.41 & \textbf{22.48}\\
Receding Hairline & 6.14 & 5.82 & 6.17 & 5.43 & 5.52 & \textbf{5.28}\\
Rosy Cheeks & 5.55 & 5.18 & 5.40 & 5.13 & {5.10} & \textbf{4.82}\\
Sideburns & 3.29 & 3.79 & 3.24 & 2.94 & 3.26 & \textbf{2.87}\\
Smiling & 8.05 & 7.18 & 8.40 & 7.21 & 7.19 & \textbf{6.74}\\
Straight Hair & 18.21 & 17.25 & 18.15 & \textbf{15.93} & 16.82 & 16.32\\
Wavy Hair & 16.53 & 15.55 & 16.19 & 13.93 & 15.28 & \textbf{13.19}\\
Wearing Earrings & 11.12 & \textbf{9.76} & 11.46 & 10.17 & 10.57 & 10.17\\
Wearing Hat & 1.15 & 1.13 & 1.08 & \textbf{0.94} & 1.14 & 0.95\\
Wearing Lipstick & 7.91 & 7.56 & 8.06 & 7.47 & 7.76 & \textbf{7.15}\\
Wearing Necklace & 13.27 & 11.90 & 13.47 & 11.61 & 11.75 & \textbf{11.32}\\
Wearing Necktie & 3.80 & 3.29 & 4.04 & 3.57 & 3.63 & \textbf{3.27}\\
Young & 13.25 & 13.40 & 13.78 & 12.26 & 12.53 & \textbf{11.83}\\
\midrule
Average & 9.62 & 8.77  & 9.53 & 8.44 & 8.73 & \textbf{8.17}\\
\bottomrule
\end{tabular}}
\end{table}

\end{document}